\documentclass{article}

\PassOptionsToPackage{numbers,compress}{natbib}
\usepackage[preprint]{neurips_2026}
\usepackage[utf8]{inputenc}
\usepackage[T1]{fontenc}
\usepackage{hyperref}
\usepackage{url}
\usepackage{booktabs}
\usepackage{amsfonts}
\usepackage{amsmath,amssymb,amsthm,mathtools}
\usepackage{nicefrac}
\usepackage{microtype}
\usepackage{float}
\usepackage{xcolor}
\usepackage{subcaption}
\usepackage{graphicx}
\usepackage{capt-of}
\usepackage{bm}
\usepackage{algorithm}
\usepackage{algpseudocode}
\usepackage[normalem]{ulem}

\title{Deep ZakaiJ: Structured Filtering for Jump-Diffusion Time Series Forecasting}
\author{Yan Leng \\
University of Texas at Austin\\
\texttt{yan.leng@mccombs.utexas.edu} \\
\And
Thibaut Mastrolia \\
University of California, Berkeley\\
\texttt{mastrolia@berkeley.edu} \\
\And
Hao Wang \\
University of California, Berkeley\\
\texttt{haowang013@berkeley.edu}}

\newcommand{\E}{\mathbb{E}}
\newcommand{\R}{\mathbb{R}}
\newcommand{\dd}{\mathrm{d}}
\definecolor{lengCommentColor}{RGB}{180,85,0}

\newtheorem{theorem}{Theorem}
\newtheorem{proposition}[theorem]{Proposition}
\newtheorem{corollary}[theorem]{Corollary}
\newtheorem{assumption}[theorem]{Assumption}

\begin{document}

\maketitle

\begin{abstract}
Time series driven by unobserved latent states frequently exhibit abrupt jump discontinuities whose timing and magnitude cannot be predicted from observed history alone. Classical jump-diffusion models offer a principled mathematical framework but assume rigid parametric forms, while recent neural jump models operate on fully observed trajectories without inferring the hidden states that govern the dynamics. We propose \textit{Deep ZakaiJ}, a latent-state model for partially observed jump-diffusion systems that embeds the Zakai nonlinear filtering equation into a neural encoder--decoder architecture. The encoder recursively updates a belief over the latent state via Strang splitting into three interpretable substeps: prior propagation, diffusion innovation, and jump innovation, yielding a differentiable, first-order-accurate approximation of the exact filtering evolution. The decoder is a structured jump-diffusion model explicitly conditioned on the filtered belief, preserving the separation between continuous dynamics and discontinuous shocks. On synthetic, financial, and oceanographic datasets, \textit{Deep ZakaiJ} improves distributional forecasts while remaining competitive in point accuracy, achieving calibrated predictive intervals and recovering interpretable latent structure in synthetic and qualitative case studies.
\end{abstract}

\section{Introduction}
Empirical time series rarely evolve smoothly; instead, they are frequently disrupted by abrupt jumps \citep{aminikhanghahi2017survey, short2013improved}. These discontinuities span diverse systems, from financial assets suffering sudden price shocks due to macroeconomic news \citep{merton1976option}, to oceanic wave heights spiking drastically during severe storms \citep{young2011global, young1988parametric}. Because these extreme dislocations are usually driven by latent states, predicting such trajectories remains a major open challenge.

Classical stochastic models, such as Merton’s jump-diffusions \citep{merton1976option}, stochastic volatility models \citep{heston1993closed} or more general L\'evy processes \citep{bertoin1996levy}, offer a principled mathematical foundation. While models like Heston explicitly introduce a latent volatility state to drive dynamics, they are limited by rigid parametric assumptions and are structurally restricted from capturing complex latent structure. Beyond standard deep learning time-series models \citep{yu2019review, benidis2022deep, lim2021time}, a growing line of work models sequential data through continuously evolving latent-state dynamics (e.g., Latent ODEs \citep{rubanova2019latent} and Neural CDEs \citep{kidger2020neural}), which propagate hidden states via learned vector fields and approximate the latent posterior through amortized variational inference, but treat latent states as implicit representations rather than explicit stochastic objects, lacking a principled Bayesian update mechanism. Models like CRU \citep{schirmer2022modeling} and NCDSSM \citep{ansari2023neural} close this gap by embedding continuous-discrete Kalman filtering within deep sequence models, yet remain restricted to linear-Gaussian dynamics. For time series exhibiting jump discontinuities, where jump times are totally inaccessible stopping times that cannot be predicted from the historical filtration \citep{protter2012stochastic, ankirchner2008filtration}, correctly updating latent beliefs demands a nonlinear filtering mechanism that current frameworks cannot provide.

In this work, we propose \textit{Deep ZakaiJ}, a framework that embeds the Zakai nonlinear filtering equation into a neural architecture for structured inference in partially observed jump-diffusion systems. Our main contributions are:
\begin{enumerate}
\item[(i)] A Zakai encoder that discretizes the Zakai equation via Strang splitting into three interpretable steps: prior propagation, diffusion innovation, and jump innovation. The jump innovation step performs likelihood reweighting against a mixture of no-jump and jump-occurred hypotheses, enabling the posterior to react sharply to dislocations. We prove that this encoder achieves first-order global accuracy relative to the exact Zakai evolution.
\item[(ii)] A structured jump-diffusion decoder that parameterizes drift, volatility, jump intensity, and jump-size distribution with neural networks conditioned on both the observation and the filtered latent belief, learning the predictable compensator so that the residual reduces to a zero-mean martingale, and coupling forecasting directly to the latent filtering state.
\item[(iii)] Experiments on synthetic, financial, and oceanographic datasets show that \textit{Deep ZakaiJ} improves distributional quality while remaining competitive in point accuracy, achieving well-calibrated predictive intervals near the nominal 90\% level and recovering interpretable latent structure in synthetic and qualitative case studies.
\end{enumerate}

\section{Related work}
Time series forecasting has evolved from classical methods such as ARIMA \citep{ho1998use} and LSTM \citep{yu2019review} to Transformer variants (e.g., PatchTST \citep{nie2023a}, iTransformer \citep{liu2024itransformer}), probabilistic models like DeepAR \citep{salinas2020deepar}, and pretrained foundation models (e.g., Chronos \citep{ansari2024chronos}, TimesFM \citep{das2024a}, Moirai-MoE \citep{liu2025moirai}, Lag-Llama \citep{rasul2023lagllama}). While increasingly powerful, these approaches typically operate in discrete time without explicit mechanisms for jump discontinuities or latent regime inference.

\textbf{Latent-state models.} Early variational approaches such as VRNN \citep{chung2015recurrent}, Deep Kalman Filters \citep{krishnan2015deep}, DVBF \citep{karl2017deep}, and KVAE \citep{fraccaro2017disentangled} introduce stochastic latent variables into sequential models but train posteriors via ELBOs rather than filtering equations. SRNN \citep{fraccaro2016sequential} stacks a deterministic RNN with a stochastic latent chain and uses a backward pass for smoothing, SNLDS \citep{dong2020collapsed} collapses discrete switching variables analytically and amortizes the continuous latent posterior to capture regime changes. Deterministic structured SSMs (S4 \citep{gu2022efficiently}, Mamba \citep{gu2024mamba}) achieve efficient long-range modeling but do not perform stochastic inference. Continuous-time extensions including Latent SDEs \citep{li2020scalable} and GRU-ODE-Bayes \citep{de2019gru} approximate the latent posterior through amortized variational inference, treating latent states as implicit representations. CRU \citep{schirmer2022modeling} and NCDSSM \citep{ansari2023neural} go further by embedding exact continuous-discrete Kalman filtering into deep models, demonstrating clear benefits of filtering, but remain restricted to linear-Gaussian dynamics. Particle-filter-based methods such as FIVO \citep{maddison2017filtering}, VSMC \citep{naesseth2018variational}, SIXO \citep{lawson2022sixo} provide asymptotically exact Bayesian inference through sequential Monte Carlo, while PF-RNN \citep{ma2020particle} embeds a differentiable particle filter directly into a recurrent cell. Recent extensions handle regime-switching systems \citep{brady2024regime}. However, these methods rely on sample-based approximations and do not exploit the continuous structure of filtering equations.

\textbf{Neural-jump models.} Neural Jump SDEs \citep{jia2019neural} model temporal point processes via piecewise-continuous latent trajectories where discrete events trigger state-dependent jumps. NJ-ODE \citep{herrera2021neural} learns the conditional expectation of a stochastic process continuously in time, with convergence guarantees, by evolving a hidden state via Neural ODEs \citep{chen2018neural} and applying jump updates at each irregular observation. NeuralMJD \citep{gao2025neural} parameterizes a non-stationary Merton jump-diffusion directly from data, and NJDTPP \citep{zhang2024neural} unifies intensity modeling through neural jump-diffusion SDEs. While NJ-ODE performs a form of implicit filtering through its jump updates, none of these models embeds an explicit nonlinear filtering equation over a latent state to jointly infer hidden regimes and forecast jump-diffusion dynamics.

\textbf{Nonlinear filtering.} The Zakai equation \citep{zakai1969optimal} provides an unnormalized formulation of the nonlinear filtering problem that is naturally amenable to operator splitting. The classical splitting method \citep{bensoussan1990approximation,gobet2006discretization,derchu2020bayesian} decomposes the Zakai PDE into independent substeps, with recent extensions handling joint diffusion and point-process observations \citep{zhang2024splitting}. The energy-based deep splitting method \citep{baagmark2023energy} approximates high-dimensional splitting substeps with neural networks. \textit{Deep ZakaiJ} builds on these numerical foundations but differs in purpose: rather than solving a filtering PDE in isolation, it embeds the discretized Zakai equation into an end-to-end architecture, jointly performing nonlinear filtering over a hidden regime and structured conditional jump-diffusion prediction.

\section{Background}
\label{sec:background}
Continuous-time stochastic models are widely used for sequential data. However, standard diffusion processes driven solely by Brownian motion cannot capture the abrupt dislocations frequently observed in practice. A natural extension is the jump-diffusion model \citep{black1973pricing,kou2002jump}
\begin{equation*}
\dd X_t
=
\mu(t,X_t)\dd t
+
\sigma(t,X_t)\dd W_t
+
\int_{\R}\gamma(t,z,X_{t^-})\,N(\dd t,\dd z),
\end{equation*}
where $W_t$ is a Brownian motion, the drift $\mu$ describes the local deterministic trend, and the diffusion coefficient $\sigma$ scales the continuous random fluctuations. $N(\dd t,\dd z)$ is a Poisson random measure with compensator $\nu(\dd z)\dd t$. Over a short interval, the Poisson term records whether a jump with mark $z$ occurs, the L\'evy measure $\nu$ determines the expected jump frequency and mark distribution, and $\gamma$ maps the jump mark to the realized state displacement. The diffusion term therefore explains continuous variation, while the Poisson integral captures rare but substantial moves \citep{kou2002jump}.

In practical forecasting problems, even the jump-diffusion dynamics of the observed series are rarely autonomous: they are often driven by regimes, stochastic volatility factors, or other latent mechanisms. This leads naturally to a partially observed coupled jump-diffusion, in which the observed process $X_t$ is driven by an unobserved state $\Theta_t$ and forecasting must combine state inference with forward simulation \citep{KurtzOcone1988FilteredMP,BandiniCossoFuhrmanPham2018-POC-BSDE},
\begin{align}
\dd X_t
&=
\mu\!\left(t,X_t,\Theta_t\right)\dd t
+ \sigma\!\left(t,X_t,\Theta_t\right)\dd W_t^{X,\mathbb P}
+ \int_{\R}\gamma\!\left(t,z,X_{t^-},\Theta_t\right)\widetilde N^{X,\mathbb P}(\dd t,\dd z),
\\
\dd \Theta_t
&=
a\!\left(t,\Theta_t\right)\dd t
+ b\!\left(t,\Theta_t\right)\dd W_t^{\Theta,\mathbb P}
+ \int_{\R}k\!\left(t,\Theta_{t^-},z\right)\widetilde N^{\Theta,\mathbb P}(\dd t,\dd z),
\label{eq:theta-dynamics}
\end{align}
where the dynamics are written under the physical measure $\mathbb P$, and $
\widetilde N^{X,\mathbb P}(\dd t,\dd z)=N^{X,\mathbb P}(\dd t,\dd z)-\nu^X(\dd z)\dd t$ and $
\widetilde N^{\Theta,\mathbb P}(\dd t,\dd z)=N^{\Theta,\mathbb P}(\dd t,\dd z)-\nu^\Theta(\dd z)\dd t.
$
The coefficients $a$, $b$, and $k$ govern the latent dynamics in a similar manner. Crucially, this formulation preserves the structural advantages of jump-diffusion modeling while operating in the partially observed regime. It thus serves as a powerful counterpart to recent neural jump-process architectures, which fundamentally restrict their scope to fully observed trajectories \citep{jia2019neural,zhang2024neural,gao2025neural}. Because $\Theta_t$ is not observable, forecasting cannot condition on the latent state directly. The relevant predictive state is instead its conditional law together with a low-dimensional summary statistic,
\begin{equation*}
\pi_t(\dd \theta)
=
\mathbb{P}\!\left(\Theta_t \in \dd \theta \mid \mathcal{F}_t^X\right),
\qquad
\beta_t
=
\int \varphi(\theta)\,\pi_t(\dd \theta),
\end{equation*}
where $\mathcal{F}_t^X=\sigma\{X_s:s\le t\}$ denotes the observation filtration. Here $\pi_t$ is the nonlinear filter, $\beta_t$ is the belief feature exposed to the forecasting model, and $\varphi$ is a smooth bounded test function, which can extract a low-dimensional summary of the posterior.
For jump-diffusion observations, it is often more convenient to change to a reference measure and then study the corresponding unnormalized conditional law \citep{KurtzOcone1988FilteredMP,CeciColaneri2014ZakaiJumps}. The next two standard results justify that construction.

\begin{proposition}[Change of measure and Zakai representation]
\label{prop:zakai-representation}
Assume there exist predictable processes $\theta_0(t)\in\R$ and $\theta_1(t,z)\in(-\infty,1)$ such that
\begin{equation*}
\sigma(t,X_t,\Theta_t)\,\theta_0(t)
+
\int_{\R}\gamma(t,z,X_{t^-},\Theta_t)\,\theta_1(t,z)\,\nu^X(\dd z)
=
\mu(t,X_t,\Theta_t).
\end{equation*}
Define the density process
\begin{equation*}
\begin{aligned}
Z_t
=
\exp\Bigg(
&-\int_0^t \theta_0(s)\,\dd W_s^{X,\mathbb P}
- \frac12\int_0^t \theta_0(s)^2\,\dd s +\int_0^t\int_{\R}\ln\!\bigl(1-\theta_1(s,z)\bigr)\,\widetilde N^{X,\mathbb P}(\dd s,\dd z) \\
&+\int_0^t\int_{\R}\bigl[\ln(1-\theta_1(s,z))+\theta_1(s,z)\bigr]\nu^X(\dd z)\,\dd s
\Bigg),
\end{aligned}
\end{equation*}
and let $\frac{\dd \mathbb Q}{\dd \mathbb P}\big|_{\mathcal F_t}=Z_t$. Then
$
\dd W_t^{X,\mathbb Q}
=
\dd W_t^{X,\mathbb P}+\theta_0(t)\,\dd t$ and $
\widetilde N^{X,\mathbb Q}(\dd t,\dd z)
=
\widetilde N^{X,\mathbb P}(\dd t,\dd z)+\theta_1(t,z)\,\nu^X(\dd z)\dd t$, 
so the observation dynamics become driftless under $\mathbb Q$:
\begin{equation*}
\dd X_t
=
\sigma(t,X_t,\Theta_t)\,\dd W_t^{X,\mathbb Q}
+
\int_{\R}\gamma(t,z,X_{t^-},\Theta_t)\,\widetilde N^{X,\mathbb Q}(\dd t,\dd z).
\end{equation*}
For any smooth bounded test function $\varphi$, the corresponding unnormalized conditional law is
\[
q_t(\varphi)
=
\E_{\mathbb Q}\!\left[
\varphi(\Theta_t) Z_t^{-1}\mid \mathcal F_t^X
\right].
\]
\end{proposition}

Regarding Equation~\eqref{eq:theta-dynamics}, the infinitesimal generator acting on a smooth bounded test function $\varphi$ is
\begin{equation*}
\begin{aligned}
(\mathcal L_t\varphi)(\theta)
=
&\,a(t,\theta)\,\partial_\theta\varphi(\theta)
+ \frac12 b^2(t,\theta)\,\partial_{\theta\theta}^2\varphi(\theta) \\
&+\int_{\R}
\Big[
\varphi\!\big(\theta+k(t,\theta,z)\big)
- \varphi(\theta)
- k(t,\theta,z)\,\partial_\theta\varphi(\theta)
\Big]\nu^\Theta(\dd z).
\end{aligned}
\end{equation*}
Under Proposition~\ref{prop:zakai-representation}, the unnormalized filter evolves in weak form as
\begin{equation}
\begin{aligned}
\dd q_t(\varphi)
&=
q_t\!\left(\mathcal L_t\varphi\right)\dd t
+
q_t(\varphi)\,\theta_0(t)\dd W_t^{X,\mathbb Q}
+
\int_{\R}
q_{t^-}(\varphi)\,\hat{\theta}_1(t,z)\,
\widetilde N^{X,\mathbb Q}(\dd t,\dd z),
\end{aligned}
\label{eq:zakai-weak}
\end{equation}
where $\theta_0$ is the diffusion likelihood-ratio coefficient from Proposition~\ref{prop:zakai-representation} and $\hat{\theta}_1 = \theta_1/(1-\theta_1)$ is the transformed jump coefficient.

\begin{theorem}[Existence and uniqueness of the jump-diffusion Zakai equation]
\label{thm:zakai-wellposed}
Under standard assumptions for jump-diffusion observations \citep{CeciColaneri2014ZakaiJumps}, there exists a unique solution to Zakai equation~\eqref{eq:zakai-weak} in the corresponding measure-valued class.
\end{theorem}

Motivated by Proposition~\ref{prop:zakai-representation} and Zakai equation~\eqref{eq:zakai-weak}, we parameterize the latent belief by an unnormalized density $q_k(\theta)\approx q_{t_k}(\theta)$ on a fixed latent grid, and its normalized counterpart by $\pi_k(\theta)=q_k(\theta)/\int q_k(\vartheta)\,\mathrm{d}\vartheta$. This parametrization is crucial because the Zakai dynamics are multiplicative in the observation innovation, which makes them particularly amenable to operator splitting and differentiable likelihood reweighting.

\section{Methods}
\label{sec:method}
Given a past observation window $\{X_{t_0},\dots,X_{t_M}\}$, our goal is to infer the current latent belief state and then predict the future path $\{X_{t_{M+1}},\dots,X_{t_{M+N}}\}$ over a horizon of $N$ steps. On a discrete grid we write $\Delta X_k=X_{t_{k+1}}-X_{t_k}$, use the shorthand $\pi_k \equiv \pi_{t_k}$, and define $\beta_k=\int \varphi(\theta)\,\pi_k(\dd\theta)$. Since $\Theta_t$ is unobserved, our model operates through the belief summary $\beta_k$ and conditions on candidate grid values $\theta_j$ for filtering and forecasting. At a high level, the model consists of a Zakai encoder $\mathcal{E}_{\psi}$ that recursively updates the unnormalized posterior belief, and a structured jump-diffusion decoder that parameterizes the conditional increment distribution:
\begin{align*}
q_{k+1}
&=
\mathcal{E}_{\psi}\!\left(q_k,\Delta X_k; X_{t_k},\beta_k\right),
\\
p_{\phi}\!\left(\Delta X_{k+1}\mid X_{0:k}\right)
&=
\int
p_{\phi}\!\left(\Delta X_{k+1}\mid t_k,X_{t_k},\beta_k,\theta\right)\,
\pi_k(\theta)\,\dd \theta.
\end{align*}
This encoder–decoder decomposition is the key design principle of our framework. An overview of the full architecture is shown in Figure~\ref{fig:architecture}. It retains the uncertainty-aware structure of nonlinear filtering, but replaces analytically intractable operators with differentiable approximations that can be trained end-to-end.
\begin{figure}[t]
    \centering
    \includegraphics[width=0.85\linewidth]{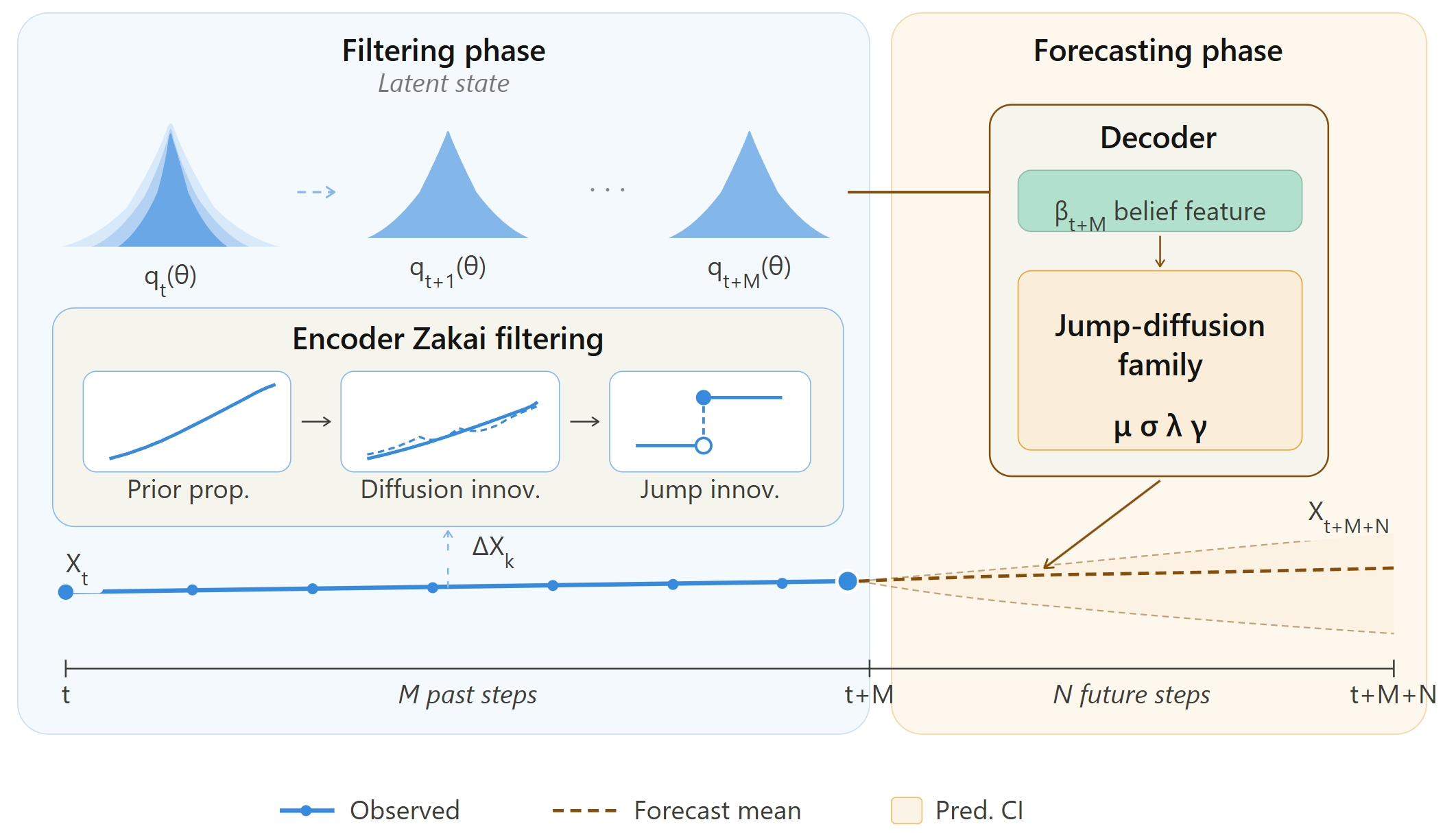}
    \caption{Overview of \textit{Deep ZakaiJ}. The encoder (left) updates the unnormalized posterior $q_t(\theta)$ over $M$ past steps via Strang splitting into prior propagation, diffusion innovation, and jump innovation. The resulting belief $\beta_{t+M}$ conditions a jump-diffusion decoder (right) for $N$-step forecasting.}
    \label{fig:architecture}
\end{figure}

\subsection{Deep ZakaiJ encoder}
\label{sec:encoder}
Our encoder is a discretization of the Zakai equation rather than an ad hoc recurrent state update. Inspired by splitting-up methods for nonlinear filtering and recent learned variants \citep{bensoussan1990approximation,zhang2024splitting,baagmark2023energy,spiteri2025beyond}, we decompose one update into latent prior propagation, diffusion innovation, and jump innovation. Let $h=\Delta t/2$. The resulting palindromic Strang-style update is
\begin{equation}
q_{k+1}
\approx
\mathcal{C}_{h}
\circ
\mathcal{B}_{h}
\circ
\mathcal{A}_{\Delta t}
\circ
\mathcal{B}_{h}
\circ
\mathcal{C}_{h}(q_k),
\label{eq:strang-update}
\end{equation}
where $\mathcal{A}_{\Delta t}$ propagates the latent prior, $\mathcal{B}_{h}$ incorporates the continuous part of the observation, and $\mathcal{C}_{h}$ handles jumps. This factorization makes the update interpretable and numerically stable: each substep targets one physical mechanism, and the innovation steps are likelihood reweightings that preserve positivity after normalization. Both $\mathcal{B}$ and $\mathcal{C}$ share the same observed increment $\Delta X_k$, reflecting the fact that the splitting decomposes a single observation likelihood into its continuous and jump components over half-steps rather than applying two separate likelihoods.



\textbf{$\mathcal{A}$-step: prior propagation.}
Let $\mathcal{L}_{\Theta}^{\ast}$ denote the forward operator associated with the latent process in \eqref{eq:theta-dynamics}. The exact prior evolution over one step is $S_{\Delta t}^{\mathcal{A}} = \exp(\Delta t \mathcal{L}_{\Theta}^{\ast})$, but in general it is not available in closed form. We therefore use a numerical propagator $\widetilde{S}_{\Delta t}^{\mathcal{A}}$ defined on a latent grid $\{\theta_j\}_{j=1}^G$:
\begin{equation*}
\widetilde{S}_{\Delta t}^{\mathcal{A}}[q](\theta_j)
=
\sum_{i=1}^{G}
K_{\Delta t}(\theta_j \mid \theta_i)\,q(\theta_i)\,\Delta \theta,
\qquad
K_{\Delta t}(\cdot \mid \theta_i)\ge 0,
\quad
\sum_{j=1}^{G}K_{\Delta t}(\theta_j \mid \theta_i)\Delta \theta = 1.
\end{equation*}
where \(K_{\Delta t}\) denotes the Markov transition kernel induced by a consistent discretization of the latent generator \(\mathcal{L}_{\Theta}^{\ast}\). To account for numerical bias and model mismatch, we add a learned residual correction $R_{\psi}^{\mathcal{A}}(q_k;\Delta t)$ constrained to have zero total mass, i.e., $\sum_{j=1}^{G} R_{\psi}^{\mathcal{A}}(q_k;\Delta t)(\theta_j)\,\Delta\theta = 0$. The $\mathcal{A}$-step output is
\begin{equation*}
q^{\mathcal{A}}(\theta_j) = \frac{\left[\widetilde{S}_{\Delta t}^{\mathcal{A}}[q_k](\theta_j) + R_{\psi}^{\mathcal{A}}(q_k;\Delta t)(\theta_j)\right]^{+}}{\sum_{l=1}^{G}\left[\widetilde{S}_{\Delta t}^{\mathcal{A}}[q_k](\theta_l) + R_{\psi}^{\mathcal{A}}(q_k;\Delta t)(\theta_l)\right]^{+}\Delta\theta},
\end{equation*}
where $[\cdot]^{+} = \max(\cdot, 0)$ clips negative values. This preserves the physical prior while allowing the encoder to learn systematic corrections that are useful for prediction.

\textbf{$\mathcal{B}$-halfstep: diffusion innovation.}
Conditioned on a candidate latent value $\theta$, the decoder specifies a local Gaussian approximation for the continuous part of the observed increment,
\begin{equation*}
\Delta X_k \mid \theta
\approx
\mathcal{N}\!\left(
\mu_{\phi}(t_k,X_{t_k},\beta_k,\theta)\Delta t,\,
\sigma_{\phi}^2(t_k,X_{t_k},\beta_k,\theta)\Delta t
\right).
\end{equation*}
Using a half-step $h=\Delta t/2$, we define the diffusion likelihood factor
\begin{equation*}
\ell_{\mathcal{B},k}(\theta_j)
=
\mathcal{N}\!\left(
\Delta X_k;\,
\mu_{\phi}(t_k,X_{t_k},\beta_k,\theta_j)h,\,
\sigma_{\phi}^2(t_k,X_{t_k},\beta_k,\theta_j)h
\right).
\end{equation*}
The $\mathcal{B}$-update is a Bayes reweighting:
\begin{equation*}
q^{\mathcal{B}}(\theta_j)
=
\frac{q^{\mathcal{A}}(\theta_j)\,\ell_{\mathcal{B},k}(\theta_j)}
{\sum_{l=1}^{G} q^{\mathcal{A}}(\theta_l)\,\ell_{\mathcal{B},k}(\theta_l)\,\Delta\theta}.
\end{equation*}
This step sharpens the posterior in regions of latent space that explain the continuous component of the observed move.

\textbf{$\mathcal{C}$-halfstep: jump innovation.}
To capture abrupt moves, we augment the diffusion likelihood with an at-most-one-jump approximation, valid for small $h$. Let $\lambda_{\phi}(t_k,X_{t_k},\beta_k,\theta_j)$ be the jump intensity and let $\gamma_{\phi}(t_k,z,X_{t_k},\beta_k,\theta_j)$ be the jump-size map for mark $z$. The half-step jump likelihood is
\begin{align}
\ell_{\mathcal{C},k}(\theta_j)
&=
e^{-\lambda_{\phi}h}
\mathcal{N}\!\left(
\Delta X_k;\,
\mu_{\phi}h,\,
\sigma_{\phi}^2 h
\right)
\nonumber\\
&\quad
+
h\,e^{-\lambda_{\phi}h}
\int
\mathcal{N}\!\left(
\Delta X_k;\,
\mu_{\phi}h+\gamma_{\phi}(t_k,z,X_{t_k},\beta_k,\theta_j),\,
\sigma_{\phi}^2 h
\right)\nu^X(\mathrm{d} z),
\label{eq:c-likelihood}
\end{align}
where, to lighten notation, the dependence of $\mu_{\phi}$, $\sigma_{\phi}$, and $\lambda_{\phi}$ on $(t_k,X_{t_k},\beta_k,\theta_j)$ is suppressed inside the Gaussian terms. The integral over  $\nu^X(\mathrm{d}z)$ is retained as it averages over jump sizes rather than the latent grid. The $\mathcal{C}$-update again takes the form of a normalized reweighting:
\begin{equation*}
q^{\mathcal{C}}(\theta_j)
=
\frac{q^\mathcal{B}(\theta_j)\,\ell_{\mathcal{C},k}(\theta_j)}
{\sum_{l=1}^{G} q^\mathcal{B}(\theta_l)\,\ell_{\mathcal{C},k}(\theta_l)\,\Delta\theta}.
\end{equation*}
This construction is simple but important: it lets the posterior react much more strongly to large observed dislocations than a purely Gaussian filter would, while keeping the update differentiable and numerically stable.

After the final innovation step, the updated belief feature is extracted over the latent grid as
\begin{equation*}
\beta_{k+1}
=
\sum_{j=1}^{G}
\varphi(\theta_j)\,
\frac{q_{k+1}(\theta_j)}{\sum_{l=1}^{G} q_{k+1}(\theta_l)\,\Delta\theta}\,
\Delta\theta.
\end{equation*}
The encoder is thus fully differentiable and remains interpretable at each step: $\mathcal{A}$ propagates prior belief, $\mathcal{B}$ explains continuous shocks, and $\mathcal{C}$ explains jump shocks.

\subsection{Structured decoder and joint training}
\label{sec:decoder}

The decoder specifies the conditional one-step law of future increments given the current observation, the belief feature, and a candidate latent state. A critical methodological choice is to constrain this law within a structured jump-diffusion family rather than replacing it with an unconstrained black-box transition density. This explicitly disentangles the continuous volatility from discontinuous shocks, and intrinsically ties the forecasting module to the same stochastic structure as the encoder. Relative to prior neural jump models \citep{jia2019neural,gao2025neural,zhang2024neural}, our decoder is fundamentally coupled to a dynamically filtered latent state rather than operating only on fully observed trajectories or event labels.

At each time step, the decoder outputs the local coefficients $\mu_{\phi}(t,x,\beta,\theta)$, $\sigma_{\phi}(t,x,\beta,\theta)$, $\lambda_{\phi}(t,x,\beta,\theta)$, and $\gamma_{\phi}(t,z,x,\beta,\theta)$. To balance expressivity with numerical stability, we parameterize these coefficients using a compact, context-conditioned neural network. Strict positivity constraints are imposed on the volatility and jump intensity, and the jump-size distribution is assigned a structured parameterization.

To formally manage the infinite-activity nature of general jump-diffusions within a discrete neural architecture, we rely on the following standard result for Lévy processes:

\begin{theorem}[Small-jump Gaussian approximation \citep{asmussen2001approximations}]
\label{thm:gaussian-approx}
Fix a truncation threshold $\varepsilon>0$ and decompose the jump measure into small- and large-jump parts. Over short horizons, the compensated small-jump component can be approximated by a Gaussian term with matching moments. In our decoder, this yields effective coefficients
\(
\tilde\mu_\phi, \tilde\sigma_\phi^2,
\)
which absorb the contribution of jumps with magnitude below $\varepsilon$.
\end{theorem}

By virtue of Theorem~\ref{thm:gaussian-approx}, the decoder absorbs infinite-activity micro-jumps into the effective continuous coefficients $\tilde{\mu}_{\phi}$ and $\tilde{\sigma}_{\phi}^2$. We therefore treat only macroscopic large jumps explicitly in the $\mathcal{C}$-step. During filtering over the observed history, the encoder applies the full $\mathcal{A}/\mathcal{B}/\mathcal{C}$ splitting update. In forecasting, future observations are strictly unavailable, meaning the diffusion and jump likelihoods become uninformative. Consequently, the posterior is propagated solely by the $\mathcal{A}$-step prior dynamics, and the belief feature is extracted from the resulting density. This constructs a coherent forward operator on the triple $(X_k, q_k, \beta_k)$ so that the predictive uncertainty naturally expands under the learned physical prior even when no new observations arrive.

Training is performed on sliding windows partitioned into a past context segment $X_{s:s+M}$ and a future prediction segment $X_{s+M:s+M+N}$ generated by Euler--Maruyama method. Given the context, the encoder sequentially produces latent posteriors and belief features. We optimize the encoder and decoder jointly by maximizing a stepwise filtering-and-forecasting objective:
\[
\mathcal{L}(\phi,\psi)
=
\sum_{k=0}^{M+N-1}
\mathbb{E}_{\pi_k(\theta)}
\!\left[
\log p_{\phi}(\Delta X_k \mid t_k, X_{t_k}, \beta_k, \theta)
\right]
-
\sum_{k=0}^{M-1}
\mathrm{KL}\!\left(
\pi_k \;\middle\|\; \pi_k^{\mathrm{prior}}
\right),
\]
Gradients are backpropagated through all Strang splitting operations and normalizations, ensuring the encoder shapes belief states that are maximally informative for downstream forecasting.

Proposition~\ref{prop:zakai-representation} and Theorem~\ref{thm:zakai-wellposed} establish the continuous-time filtering model, while Theorem~\ref{thm:gaussian-approx} justifies our treatment of the decoder's jump measure. What remains to be rigorously bounded is the structural approximation error introduced by the discrete split encoder itself.

\begin{assumption}[$L^1$-stability]
\label{ass:l1-stability}
There exists $L \ge 0$ such that for any normalized densities $q_1$ and $q_2$,
\begin{equation*}
\left\lVert
\mathcal{S}_{\Delta t}(q_1)-\mathcal{S}_{\Delta t}(q_2)
\right\rVert_1
\le
(1+L\Delta t)
\left\lVert q_1-q_2 \right\rVert_1,
\end{equation*}
where $\mathcal{S}_{\Delta t}$ denotes the exact one-step filtering semigroup.
\end{assumption}

\begin{theorem}[Local approximation error of the split encoder]
\label{thm:local-error}
Let $\widetilde{\mathcal{S}}_{\Delta t}$ denote the discrete update in Equation~\eqref{eq:strang-update} together with the at-most-one-jump approximation in Equation~\eqref{eq:c-likelihood}. Assume Assumption~\ref{ass:l1-stability} and suppose the intensity of jumps larger than $\varepsilon$ is uniformly bounded by $\Lambda_{\varepsilon}$. Then there exist constants $C_1,C_2>0$, independent of $\Delta t$, such that for any normalized input density $q$,
\begin{equation*}
\left\lVert
\widetilde{\mathcal{S}}_{\Delta t}(q)-\mathcal{S}_{\Delta t}(q)
\right\rVert_1
\le
C_1\Delta t^2
+
C_2(\Lambda_{\varepsilon}\Delta t)^2.
\end{equation*}
\end{theorem}


The first term is the classical Strang splitting error for continuous operators \citep{bensoussan1990approximation, strang1968construction}, the second arises from truncating multiple jumps per half-step, controlled by $\mathbb{P}(N_h \ge 2)=\mathcal{O}((\Lambda_{\varepsilon}h)^2)$. This bound covers the idealized operator under exact arithmetic. In practice the encoder also normalizes after each substep, but the following proposition shows this does not degrade the convergence order.
\begin{proposition}[Normalization stability]
\label{prop:norm-stability}
Let $q, \tilde{q}$ be nonnegative integrable functions with $\|q\|_1 > 0$. Then
\[
\left\lVert \frac{\tilde{q}}{\|\tilde{q}\|_1} - \frac{q}{\|q\|_1} \right\rVert_1
\le
\frac{2}{\|q\|_1}\left\lVert \tilde{q} - q \right\rVert_1.
\]
\end{proposition}

Combined with Theorem~\ref{thm:local-error}, the normalized encoder preserves first-order global accuracy provided the unnormalized mass $\|q\|_1$ remains bounded away from zero, a condition ensured in practice by positivity clipping. Remaining implementation details (finite grid, learned residuals) are validated empirically in Appendix~\ref{app:ablation}. A full proof is provided in Appendix~\ref{app:proof}.

\begin{corollary}[Global error over a fixed horizon]
\label{cor:global-error}
Fix $T>0$ and let $n=T/\Delta t$. Under the assumptions of Theorem~\ref{thm:local-error}, there exists a constant $C_T>0$, independent of $\Delta t$, such that
\begin{equation*}
\left\lVert
\widetilde{\mathcal{S}}_{\Delta t}^{\,n}(q)
-
\mathcal{S}_{\Delta t}^{\,n}(q)
\right\rVert_1
\le
C_T(1+\Lambda_{\varepsilon}^2)\Delta t.
\end{equation*}
\end{corollary}

Thus, the discrete encoder is first-order accurate over a fixed horizon despite possessing a second-order local splitting defect, providing a controlled approximation to the Zakai evolution while remaining substantially easier to train than traditional solvers.

\section{Experiments}
\label{sec:experiments}

We evaluate \textit{Deep ZakaiJ} on one synthetic benchmark and two real-world datasets, covering both controlled settings and practical forecasting scenarios. Our goal is to assess not only predictive accuracy, but also the ability to capture stochastic dynamics and jump behaviors. Implementation details, preprocessing, and hyperparameters are deferred to Appendix~\ref{app:exp-details}.


\textbf{Baselines.}
We compare \textit{Deep ZakaiJ} with a diverse set of models spanning different modeling paradigms. First, we include DLinear \citep{Zeng2022AreTE} as a simple linear reference and classical sequence models LSTM \citep{yu2019review} and DeepAR \citep{salinas2020deepar}. Transformer-based forecasting is covered by PatchTST \citep{nie2023a}. To benchmark latent-state modeling, we include LatentODE \citep{rubanova2019latent}, LatentSDE \citep{li2020scalable} and NCDSSM \citep{ansari2023neural}; these models perform latent-state inference but are restricted to continuous or linear-Gaussian dynamics without explicit jump mechanisms. For neural jump-diffusion modeling, we compare against NJ-ODE \citep{herrera2021neural} and NeuralMJD \citep{gao2025neural}, which capture jump dynamics but operate on fully observed trajectories without latent-state filtering. Finally, we include Chronos \citep{ansari2024chronos} as a representative foundation model. Additional baselines are reported in Appendix~\ref{app:exp_benchmark}.

\textbf{Evaluation metrics.}
We evaluate both point and probabilistic forecasts. 
For point accuracy, we report MAE and RMSE. 
For probabilistic quality, we use CRPS \citep{gneiting2007strictly} and average log-likelihood, which assess distributional sharpness and calibration. 
We also report empirical 90\% coverage (Cov90) to measure calibration of predictive intervals. 
Lower values indicate better performance for MAE, RMSE, and CRPS, while higher log-likelihood is preferred; for Cov90, values closer to the nominal level (0.9) indicate better calibration.

\subsection{Synthetic Data}
\textbf{Data generation.}
We evaluate \textit{Deep ZakaiJ} on a synthetic partially observed jump-diffusion process, where both the observation drift and jump intensity depend linearly on the latent state. We simulate a trajectory of 20,000 time steps and construct samples using a sliding window: conditioning on the past 300 observations, we predict the next 100 steps with stride 100. The data is split into 60\% training, 20\% validation, and 20\% testing. Full details are discussed in Appendix~\ref{app:exp_data}.
\begin{figure}[t]
\centering
\setlength{\tabcolsep}{4.5pt}

\begin{minipage}[t]{0.45\linewidth}
\vspace{0pt}
\centering

\captionsetup{type=table}
\caption{Quantitative results on synthetic dataset}
\label{tab:synthetic}

\scriptsize
\resizebox{\linewidth}{!}{
\begin{tabular}{lcc|ccc}
\toprule
Model & MAE & RMSE & CRPS & LogLik & Cov90 \\
\midrule
DLinear & 0.4095 & 0.4807 & -- & -- & -- \\
LSTM & 0.1084 & 0.1533 & -- & -- & -- \\
DeepAR & 0.6370 & 0.6788 & 0.4111 & -1.09 & 35.0 \\
PatchTST & 0.1145 & 0.1730 & -- & -- & -- \\
LatentSDE & 0.1118 & 0.1523 & 0.0846 & -3.81 & 75.7 \\
LatentODE & 0.1383 & 0.1895 & 0.1025 & -0.36 & 91.6 \\
NCDSSM & 0.1081 & 0.1507 & 0.0851 & -1.25 & 83.5 \\
NJ-ODE & 0.1042 & 0.1565 & -- & -- & -- \\
NeuralMJD & 0.0921 & 0.1560 & 0.0803 & -5.45 & 65.1 \\
Chronos & 0.1402 & 0.1854 & 0.2136 & -16.23 & 71.4 \\
\midrule
Deep ZakaiJ & \textbf{0.0846} & \textbf{0.1344} & \textbf{0.0653} & \textbf{0.78} & \textbf{89.0} \\
\bottomrule
\end{tabular}
}

\end{minipage}\hspace{0.01\linewidth}
\begin{minipage}[t]{0.5\linewidth}
\vspace{0pt}
\centering

\includegraphics[width=\linewidth]{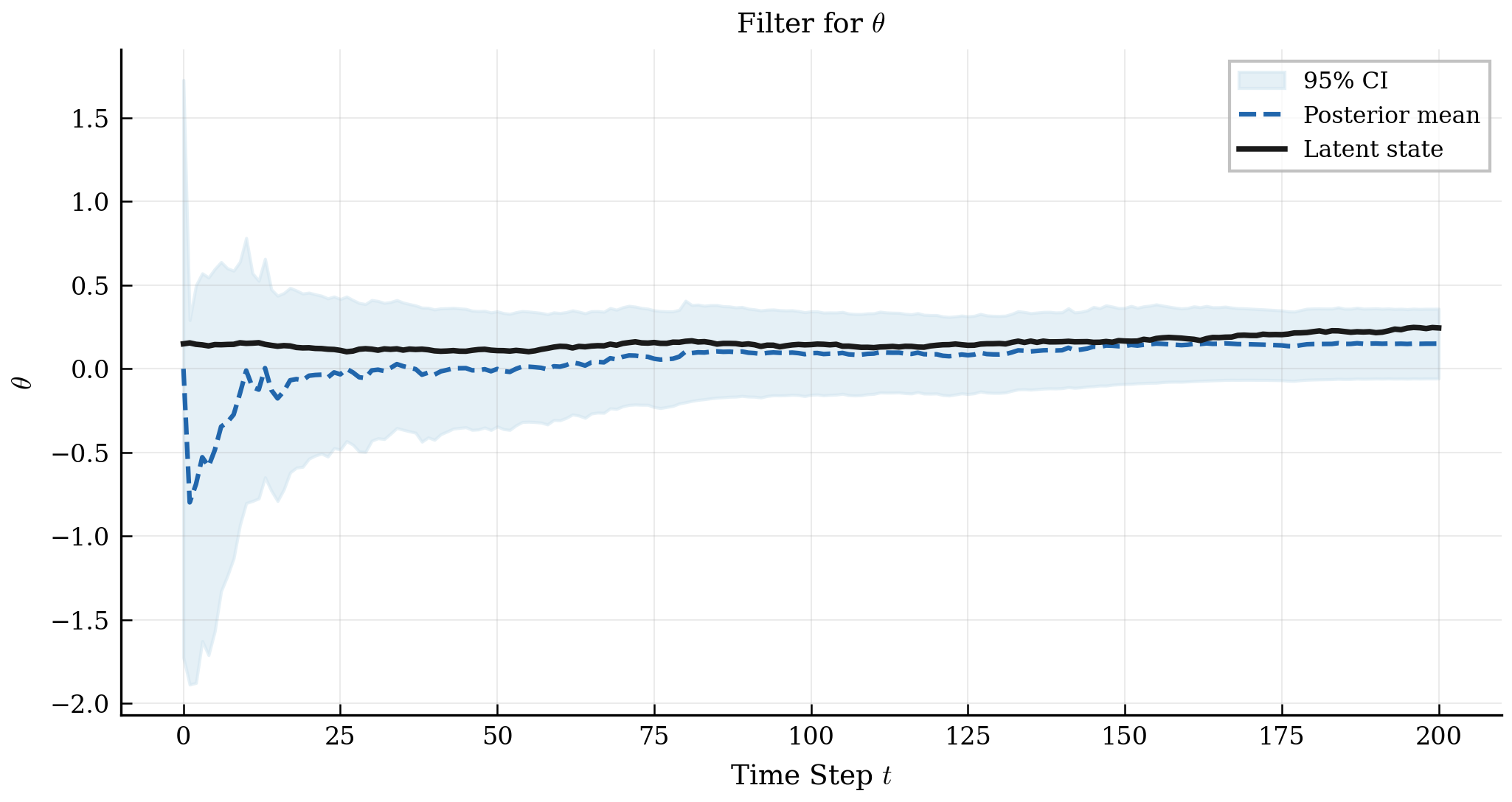}

\captionsetup{type=figure}
\caption{Filtering behavior on synthetic dataset}
\label{fig:theta}

\end{minipage}

\end{figure}

\textbf{Results.}
Since the data-generating mechanism is known, this synthetic setting tests \textit{Deep ZakaiJ} under a correctly specified linear decoder; polynomial and neural variants are deferred to later ablations. Table~\ref{tab:synthetic} shows that \textit{Deep ZakaiJ} achieves the best MAE, CRPS, and log-likelihood, while maintaining coverage close to the nominal 90\% level. Figure~\ref{fig:theta} further shows that the inferred filter tracks the latent state after a short initial transient. These results suggest that when the decoder is correctly specified, explicit filtering improves both state estimation and downstream forecasting.

\subsection{Real-world Data}
\textbf{Financial dataset (XAU/USD).}
This dataset comprises historical quotes for Gold against the U.S. Dollar, sourced from the Titan FX Research Hub \cite{titanfx_xauusd_2025}. To capture representative price movements, we resample the original one-minute OHLCV records into 10-minute intervals, focusing on the closing prices. We adopt a sliding window configuration with a lookback capacity of 300 ($M=300$) time steps and a prediction horizon of 100 ($N=100$) steps, utilizing a stride of 100 for sample generation. The data is partitioned chronologically to prevent temporal leakage: training (Jan.~2--Aug.~22, 2025), validation (Aug.~22--Oct.~10, 2025), and testing (Oct.~10--Nov.~28, 2025).

\textbf{Oceanographic dataset (NDBC Wave-Height).}
To evaluate model performance on physical environment, we utilize wave height (WVHT) observations from the NOAA/NDBC buoy station 44027 \cite{ndbc_44027_wvht_2023_2025}. This univariate series, measured in meters, spans from 2023 to 2025 at a 10-minute resolution. Consistent with the financial dataset, we implement a sliding window approach ($M=300, N=100$) with a stride of 100. The dataset is split into training (Jan.~2023--Feb.~2025), validation (Feb.--Jul.~2025), and testing (Jul.--Dec.~2025) sets, strictly following the chronological order of observations.

\begin{table}[t]
\centering
\caption{Quantitative results on the two real-world datasets.}
\label{tab:real_world}
\scriptsize
\resizebox{\linewidth}{!}{
\begin{tabular}{lcc|ccc|cc|ccc}
\toprule
& \multicolumn{5}{c|}{XAU/USD} & \multicolumn{5}{c}{NDBC Wave-Height} \\
\cmidrule(lr){2-6}\cmidrule(lr){7-11}
Model & MAE & RMSE & CRPS & LogLik & Cov90 (\%) & MAE & RMSE & CRPS & LogLik & Cov90 (\%) \\
\midrule
DLinear & 0.0097 & 0.0127 & -- & -- & -- & 0.3316 & 0.4401 & -- & -- & -- \\
LSTM & 0.0065 & 0.0095 & -- & -- & -- & 0.3387 & 0.4483 & -- & -- & -- \\
DeepAR & 0.1860 & 0.2249 & 0.1675 & -12.59 & 33.3 & 0.3909 & 0.5057 & 0.3285 & -7.25 & 37.4 \\
PatchTST & 0.0081 & 0.0114 & -- & -- & -- & 0.4096 & 0.5621 & -- & -- & -- \\
LatentSDE & 0.0063 & 0.0091 & 0.0050 & 0.62 & 55.6 & 0.3526 & 0.4628 & 0.2844 & -4.53 & 54.0 \\
LatentODE & 0.0103 & 0.0139 & 0.0075 & 2.62 & 69.7 & 0.5111 & 0.6147 & 0.3662 & -1.37 & 70.8\\
NCDSSM & 0.0063 & 0.0092 & 0.0048 & 2.92 & 77.1 & 0.3663 & 0.4858 & 0.2793 & -1.38 & 71.1\\
NJ-ODE & 0.0061 & 0.0088 & -- & -- & -- & 0.3398 & 0.4410 & -- & -- & -- \\
NeuralMJD & 0.0061 & 0.0089 & 0.0045 & 3.37 & 79.8 & 0.3038 & 0.4207 & 0.2414 & -2.06 & 54.4 \\
Chronos & 0.0062 & 0.0091 & 0.0064 & 0.65 & 72.3 & \textbf{0.3023} & 0.4142 & 0.2424 & -37.54 & 58.8 \\
\midrule
Deep ZakaiJ & \textbf{0.0056} & \textbf{0.0081} & \textbf{0.0041} & \textbf{3.65} & \textbf{92.9} & 0.3069 & \textbf{0.4084} & \textbf{0.2175} & \textbf{-0.38} & \textbf{85.3} \\
\bottomrule
\end{tabular}
}
\end{table}

\textbf{Results.} Table~\ref{tab:real_world} shows that lower values are better for MAE, RMSE, and CRPS, higher values are better for LogLik, and Cov90 is best when closest to 90. \textit{Deep ZakaiJ} improves distributional forecast quality across both domains while remaining competitive on point accuracy. On XAU/USD, it achieves the strongest performance among the evaluated baselines, notably outperforming classical SDE-based models, latent-state models and modern foundation models. On NDBC, while Chronos yields slightly lower point errors, its substantially worse LogLik and poor Cov90 ($58.8\%$) indicate weaker uncertainty calibration under oceanographic dynamics. In contrast, \textit{Deep ZakaiJ} provides a well-calibrated predictive law (Cov90: $85.3\%$), effectively bridging the gap between point accuracy and probabilistic reliability.

We further examine whether \textit{Deep ZakaiJ} learns latent structure aligned with plausible financial regimes. As shown in Fig.~\ref{fig:real_xau_theta}, the filtered belief $\theta$ co-evolves with the observed price but is not a simple echo, it shifts ahead of several trend reversals and reacts sharply to the dislocation after step $-50$, where the jump innovation concentrates posterior mass on a low regime. This illustrates how the Zakai encoder extracts a continuous, temporally coherent latent signal that conditions the downstream forecast. This pattern suggests that the filtered belief helps \textit{Deep ZakaiJ} adapt after the transition, although we treat the real-data regime interpretation as qualitative evidence rather than independent validation of the underlying market mechanism (Fig.~\ref{fig:real_xau_forecast}). For completeness, additional baseline results and ablation studies are in Appendices~\ref{app:exp_benchmark} and~\ref{app:ablation}, respectively.

\begin{figure}[t]
\centering
\begin{subfigure}[t]{0.45\linewidth}
\centering
\includegraphics[width=\linewidth]{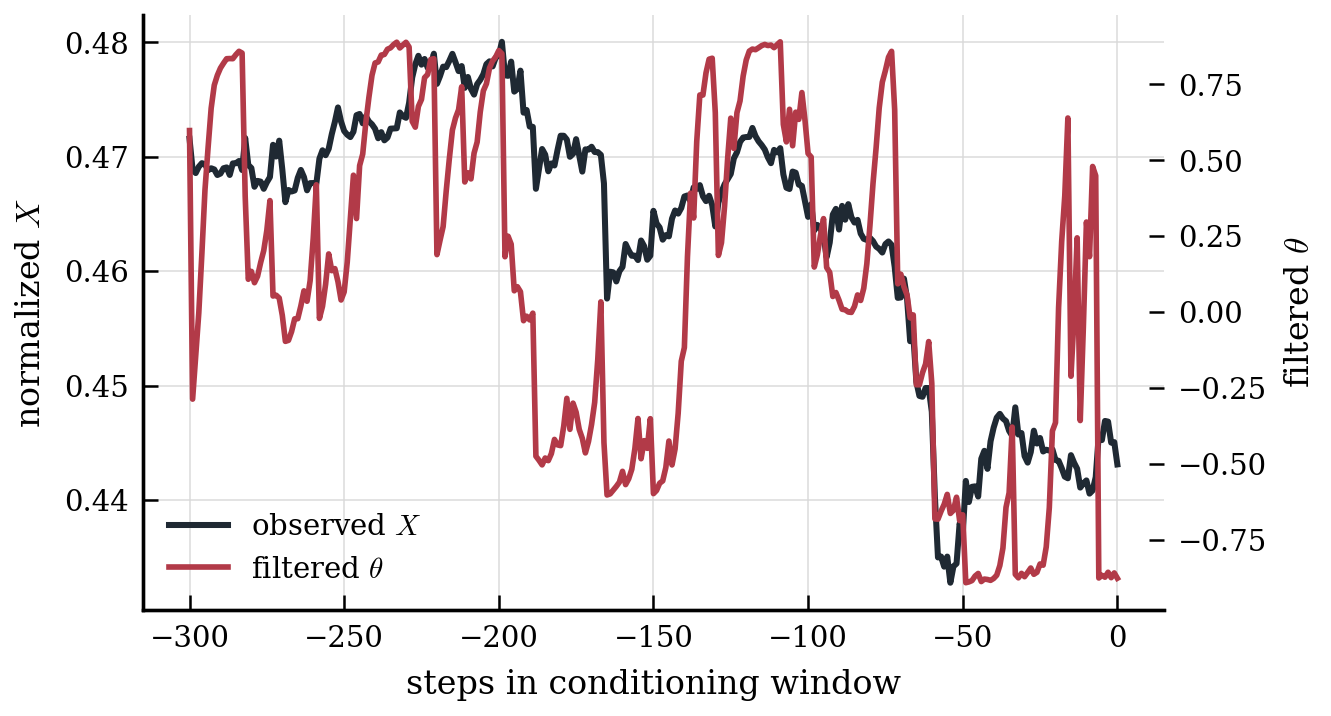}
\caption{XAU/USD conditioning window.}
\label{fig:real_xau_theta}
\end{subfigure}\hfill
\begin{subfigure}[t]{0.45\linewidth}
\centering
\includegraphics[width=\linewidth]{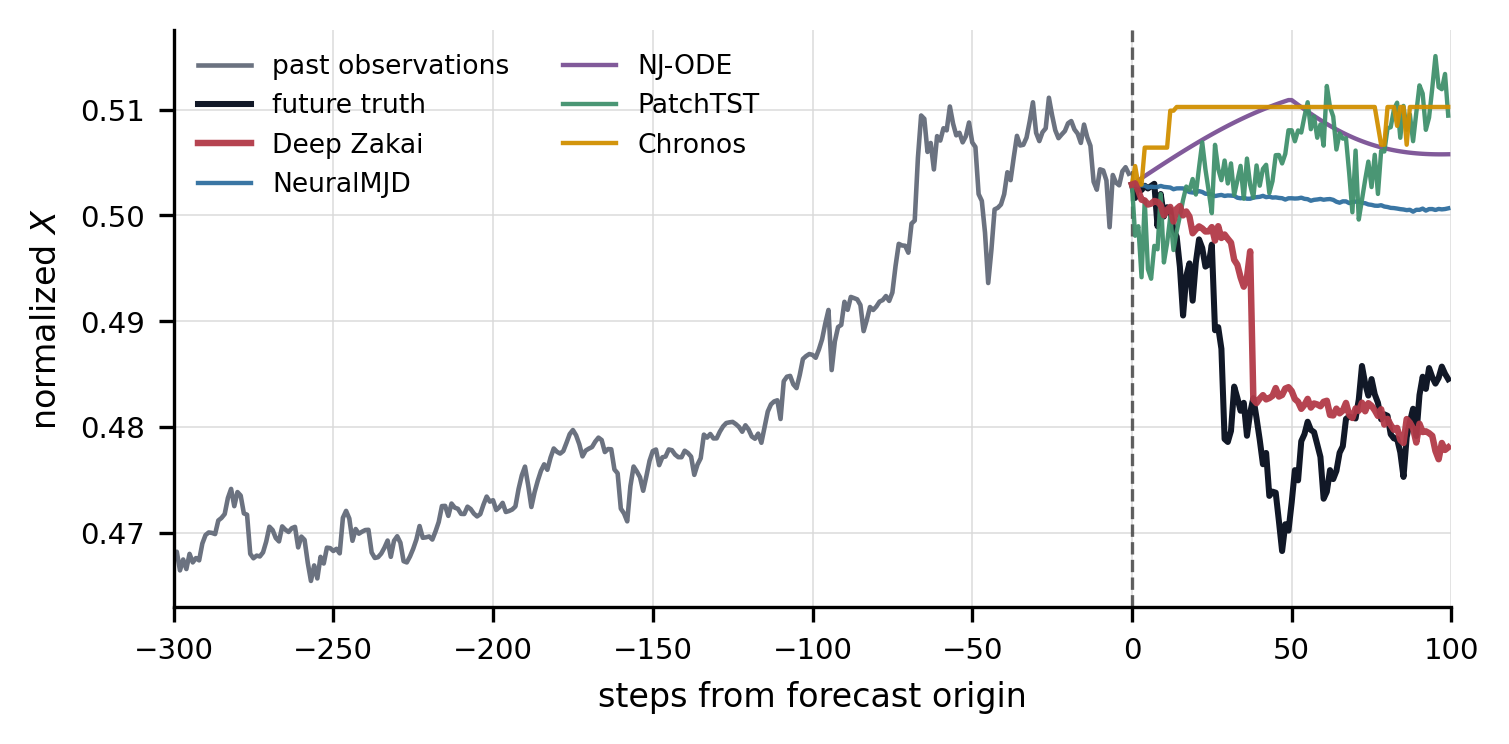}
\caption{XAU/USD 100-step forecast.}
\label{fig:real_xau_forecast}
\end{subfigure}
\caption{Qualitative examples on XAU/USD. Left: the filtered $\theta$ changes state within the observed conditioning window. Right: conditioned on the past 300 observations, \textit{Deep ZakaiJ} and selected baselines are compared over the next 100 steps.}
\label{fig:real_data_cases}
\end{figure}

\section{Conclusion}
We introduced \textit{Deep ZakaiJ}, a latent-state model for partially observed jump-diffusion systems that embeds the Zakai nonlinear filtering equation into an end-to-end trainable neural architecture. By maintaining an explicit belief over the hidden regime through principled filtering, the model separates latent-state inference from structured forecasting, learns latent representations that are interpretable in synthetic and qualitative case studies, and produces well-calibrated predictive intervals under structural shifts. Experiments show that this filtering-driven approach performs strongly among the evaluated baselines, particularly in uncertainty quantification. Future work includes extending to high-dimensional latent states.

\bibliographystyle{unsrtnat}
\bibliography{references}

\newpage
\appendix

\section{Proofs and auxiliary results}\label{app:proof}

This appendix contains the proofs and derivations omitted from the main text. Section~\ref{app:change-measure} and \ref{app:zakai} justifies the change of measure and the weak Zakai equation. Remarks on Theorems~\ref{thm:zakai-wellposed} and \ref{thm:gaussian-approx} are given in Sections~\ref{app:remark-wellposed} and \ref{app:remark-gauss}. Section~\ref{app:approx-local} and \ref{app:approx-global} proves the approximation errors in Theorem~\ref{thm:local-error} and Corollary~\ref{cor:global-error}.

\subsection{Proof of Proposition~\ref{prop:zakai-representation}}\label{app:change-measure}
\begin{proof}
    Let \(L_t = Z_t^{-1}\). Under the standard integrability conditions, the Dol\'eans exponential \(Z_t\) is a strictly positive uniformly integrable martingale, and \(1-\theta_1(t,z)>0\) ensures the jump likelihood ratio is well defined. By the Girsanov theorem for semimartingales with Brownian and jumps, the measure \(\mathbb Q\) defined by \(\dd\mathbb Q/\dd\mathbb P|_{\mathcal F_t}=Z_t\) is equivalent to \(\mathbb P\) on each \(\mathcal F_t\), and
\[
W_t^{X,\mathbb Q} = W_t^{X,\mathbb P} + \int_0^t \theta_0(s)\,\dd s,
\]
\[
\widetilde N^{X,\mathbb Q}(\dd t,\dd z) = \widetilde N^{X,\mathbb P}(\dd t,\dd z) + \theta_1(t,z)\,\nu^X(\dd z)\dd t.
\]
Substituting these into the \(\mathbb P\)-dynamics of \(X_t\) gives a drift term
\[
\mu(t,X_t,\Theta_t) - \sigma(t,X_t,\Theta_t)\theta_0(t) - \int_{\mathbb R}\gamma(t,z,X_{t^-},\Theta_t)\theta_1(t,z)\,\nu^X(\dd z),
\]
which vanishes by the matching condition in the proposition. Hence under \(\mathbb Q\),
\[
\dd X_t = \sigma(t,X_t,\Theta_t)\,\dd W_t^{X,\mathbb Q} + \int_{\mathbb R}\gamma(t,z,X_{t^-},\Theta_t)\,\widetilde N^{X,\mathbb Q}(\dd t,\dd z).
\]
For any bounded measurable \(\varphi\), Bayes' rule yields
\[
\mathbb E_{\mathbb P}[\varphi(\Theta_t)\mid \mathcal F_t^X] = \frac{\mathbb E_{\mathbb Q}[\varphi(\Theta_t)L_t\mid \mathcal F_t^X]}{\mathbb E_{\mathbb Q}[L_t\mid \mathcal F_t^X]},
\]
so the unnormalized filter is \(q_t(\varphi) = \mathbb E_{\mathbb Q}[\varphi(\Theta_t)L_t\mid \mathcal F_t^X]\).    
\end{proof}

\subsection{Derivation of the weak Zakai equation \eqref{eq:zakai-weak}}\label{app:zakai}
Under \(\mathbb Q\), the reciprocal likelihood satisfies
\[
\dd L_t = L_{t^-}\theta_0(t)\,\dd W_t^{X,\mathbb Q} + \int_{\mathbb R} L_{t^-}\,\frac{\theta_1(t,z)}{1-\theta_1(t,z)}\,\widetilde N^{X,\mathbb Q}(\dd t,\dd z).
\]
Write \(\hat{\theta}_1(t,z) = \theta_1(t,z)/(1-\theta_1(t,z))\) for the effective jump coefficient of \(L_t\). Since \(\hat{\theta}_1\) is again a bounded predictable process, the subsequent algebra is identical. The latent process \(\Theta_t\) evolves under \(\mathbb Q\) with generator \(\mathcal L_t\) unchanged from \(\mathbb P\). For a smooth bounded test function \(\varphi\),
\[
\varphi(\Theta_t) = \varphi(\Theta_0) + \int_0^t (\mathcal L_s\varphi)(\Theta_s)\,\dd s + M_t^\varphi,
\]
where \(M^\varphi\) is a \(\mathbb Q\)-martingale orthogonal to the observation noises. By It\^o's product rule for semimartingales,
\[
\dd(\varphi(\Theta_t)L_t) = \varphi(\Theta_{t^-})\,\dd L_t + L_{t^-}\,\dd\varphi(\Theta_t) + \dd[\varphi(\Theta),L]_t.
\]
The covariation \([\varphi(\Theta),L]\) vanishes because \(\Theta\) and \(W^{X,\mathbb{Q}}\) are independent under \(\mathbb{Q}\), and joint jumps contribute only through the Poisson term already captured by the \(\hat{\theta}_1\) reweighting. Applying the product rule to \(\varphi(\Theta_t)L_t\) and taking conditional expectation with respect to \(\mathcal F_t^X\) gives
\[
\dd q_t(\varphi) = q_t(\mathcal L_t\varphi)\,\dd t + q_t(\varphi)\theta_0(t)\,\dd W_t^{X,\mathbb Q}
+ \int_{\mathbb R} q_{t^-}(\varphi)\,\hat{\theta}_1(t,z)\,\widetilde N^{X,\mathbb Q}(\dd t,\dd z),
\]
which is \eqref{eq:zakai-weak}.

\subsection{Remark on Theorem~\ref{thm:zakai-wellposed}}\label{app:remark-wellposed}

Theorem~\ref{thm:zakai-wellposed} is a classical result. Under standard boundedness, Lipschitz, and positivity conditions on the coefficients and on the likelihood-ratio factors, the weak Zakai equation admits a unique c\`adl\`ag measure-valued solution; see \citet{KurtzOcone1988FilteredMP,CeciColaneri2014ZakaiJumps}. We state it as a structural fact without reproducing the proof.

\subsection{Remark on Theorem~\ref{thm:gaussian-approx}}\label{app:remark-gauss}

Theorem~\ref{thm:gaussian-approx} is a known approximation result for L\'evy processes \citep{asmussen2001approximations}. For a truncation level \(\varepsilon>0\), the compensated small-jump martingale

\[
J_t^{(\varepsilon)} = \int_0^t\int_{|z|\le \varepsilon} \gamma_{\phi}(s,z,X_{s^-},\beta_s,\theta)\,\widetilde N^X(\dd s,\dd z)
\]

has conditional mean zero and conditional quadratic variation \(\int_t^{t+h}\int_{|z|\le \varepsilon} \gamma_{\phi}^2\,\nu^X(\dd z)\dd s\). Over a short horizon \(h\), it can be replaced by a Gaussian increment with matching variance to first order. This justifies absorbing micro-jumps into the effective coefficients \(\tilde\mu_{\phi},\tilde\sigma_{\phi}^2\) in the decoder, treating only jumps larger than \(\varepsilon\) explicitly. This approximation justifies the decoder design and is independent of the encoder convergence analysis in Theorem~\ref{thm:local-error}.

\subsection{Proof of Theorem~\ref{thm:local-error}}\label{app:approx-local}
\begin{proof}
Let $h=\Delta t/2$ and define the ideal Strang update
\[
\mathcal{S}_{\Delta t}^{\mathrm{Strang}}
:=
\mathcal{C}_{h}^{\mathrm{exact}}
\circ
\mathcal{B}_{h}
\circ
\mathcal{A}_{\Delta t}
\circ
\mathcal{B}_{h}
\circ
\mathcal{C}_{h}^{\mathrm{exact}},
\]
where $\mathcal{C}_{h}^{\mathrm{exact}}$ denotes the exact jump innovation operator without the $0$--$1$ truncation. We decompose
\[
\left\lVert
\widetilde{\mathcal{S}}_{\Delta t}(q)-\mathcal{S}_{\Delta t}(q)
\right\rVert_1
\le
\underbrace{
\left\lVert
\mathcal{S}_{\Delta t}^{\mathrm{Strang}}(q)-\mathcal{S}_{\Delta t}(q)
\right\rVert_1
}_{\mathrm{(I)}}
+
\underbrace{
\left\lVert
\widetilde{\mathcal{S}}_{\Delta t}(q)-\mathcal{S}_{\Delta t}^{\mathrm{Strang}}(q)
\right\rVert_1
}_{\mathrm{(II)}}.
\]

\textbf{(I) Splitting error.}
Under the regularity assumptions ensuring the $L^1$-stability of the filtering semigroup, Strang splitting is second-order accurate for the continuous operators; see \citep{strang1968construction,bensoussan1990approximation}. Therefore,
\[
\mathrm{(I)} \le C_1 \Delta t^2
\]
for a constant $C_1>0$ independent of $\Delta t$.

\textbf{(II) Truncation error.}
Let $N_h$ be the number of jumps larger than $\varepsilon$ over a half-step. By assumption,
\[
N_h \sim \mathrm{Poisson}(\Lambda_{\varepsilon} h).
\]
The approximate operator $\widetilde{\mathcal{C}}_h$ coincides with $\mathcal{C}_{h}^{\mathrm{exact}}$ whenever $N_h \le 1$ and differs only on the event $\{N_h \ge 2\}$. Writing both operators as expectations over jump counts,
\[
\mathcal{C}_{h}^{\mathrm{exact}}(q)=\E[q^{(N_h)}],
\qquad
\widetilde{\mathcal{C}}_h(q)=\E[q^{(\min(N_h,1))}],
\]
and using that each conditional update remains a normalized density, we obtain
\[
\left\lVert
\mathcal{C}_{h}^{\mathrm{exact}}(q)-\widetilde{\mathcal{C}}_h(q)
\right\rVert_1
\le
2\,\mathbb{P}(N_h \ge 2).
\]
For a Poisson random variable,
\[
\mathbb{P}(N_h \ge 2)
=
1-e^{-\Lambda_{\varepsilon} h}(1+\Lambda_{\varepsilon} h)
\le
\frac{(\Lambda_{\varepsilon} h)^2}{2}.
\]
With $h=\Delta t/2$, this gives
\[
\left\lVert
\mathcal{C}_{h}^{\mathrm{exact}}-\widetilde{\mathcal{C}}_h
\right\rVert_{L^1}
\le
C_2(\Lambda_{\varepsilon}\Delta t)^2
\]
for a constant $C_2>0$. Since the intermediate operators $\mathcal{A}_{\Delta t}$ and $\mathcal{B}_h$ are $L^1$-stable, propagating this defect through the Strang chain preserves the same order, so
\[
\mathrm{(II)} \le C_2(\Lambda_{\varepsilon}\Delta t)^2.
\]
Combining (I) and (II) proves Theorem~\ref{thm:local-error}.
\end{proof}

\subsection{Proof of Proposition~\ref{prop:norm-stability}}\label{app:norm-stability}
\begin{proof}
Write $\bar{q} = q/\|q\|_1$ and $\bar{\tilde{q}} = \tilde{q}/\|\tilde{q}\|_1$. Then
\[
\bar{\tilde{q}} - \bar{q}
=
\frac{\tilde{q} - q}{\|\tilde{q}\|_1}
+
q\left(\frac{1}{\|\tilde{q}\|_1} - \frac{1}{\|q\|_1}\right).
\]
Taking $L^1$ norms and using $\|q/\|\tilde{q}\|_1\|_1 = \|q\|_1/\|\tilde{q}\|_1$,
\[
\|\bar{\tilde{q}} - \bar{q}\|_1
\le
\frac{\|\tilde{q} - q\|_1}{\|\tilde{q}\|_1}
+
\frac{\|q\|_1}{\|\tilde{q}\|_1}
\cdot
\frac{\bigl|\|\tilde{q}\|_1 - \|q\|_1\bigr|}{\|q\|_1}.
\]
Since $\bigl|\|\tilde{q}\|_1 - \|q\|_1\bigr| \le \|\tilde{q} - q\|_1$, the second term is bounded by $\|\tilde{q} - q\|_1 / \|\tilde{q}\|_1$. Hence
\[
\|\bar{\tilde{q}} - \bar{q}\|_1
\le
\frac{2\,\|\tilde{q} - q\|_1}{\|\tilde{q}\|_1}.
\]
Finally, $\|\tilde{q}\|_1 \ge \|q\|_1 - \|\tilde{q} - q\|_1$. For $\|\tilde{q} - q\|_1 \le \|q\|_1/2$ (which holds for sufficiently small $\Delta t$), we have $\|\tilde{q}\|_1 \ge \|q\|_1/2$, giving
\[
\|\bar{\tilde{q}} - \bar{q}\|_1
\le
\frac{2}{\|q\|_1}\|\tilde{q} - q\|_1.
\]
\end{proof}

\subsection{Proof of Corollary~\ref{cor:global-error}}\label{app:approx-global}
\begin{proof}
Let $q_k := \mathcal{S}_{\Delta t}^{\,k}(q)$ and $\tilde q_k := \widetilde{\mathcal{S}}_{\Delta t}^{\,k}(q)$. Define the one-step defect
\[
e_{k+1}
:=
\widetilde{\mathcal{S}}_{\Delta t}(\tilde q_k)
-
\mathcal{S}_{\Delta t}(\tilde q_k).
\]
Then
\[
\tilde q_{k+1}-q_{k+1}
=
\mathcal{S}_{\Delta t}(\tilde q_k)-\mathcal{S}_{\Delta t}(q_k)+e_{k+1}.
\]
By Assumption~\ref{ass:l1-stability} and Theorem~\ref{thm:local-error},
\[
\left\lVert \tilde q_{k+1}-q_{k+1} \right\rVert_1
\le
(1+L\Delta t)\left\lVert \tilde q_k-q_k \right\rVert_1
+
C_1\Delta t^2
+
C_2(\Lambda_{\varepsilon}\Delta t)^2.
\]
Iterating this recursion with $\tilde q_0=q_0$ yields
\[
\left\lVert \tilde q_n-q_n \right\rVert_1
\le
\sum_{j=0}^{n-1}
(1+L\Delta t)^{n-1-j}
\bigl(
C_1\Delta t^2 + C_2(\Lambda_{\varepsilon}\Delta t)^2
\bigr).
\]
Using $(1+L\Delta t)^n \le e^{LT}$ and $n=T/\Delta t$, we obtain
\[
\left\lVert
\widetilde{\mathcal{S}}_{\Delta t}^{\,n}(q)
-
\mathcal{S}_{\Delta t}^{\,n}(q)
\right\rVert_1
\le
C_T(1+\Lambda_{\varepsilon}^2)\Delta t
\]
for some constant $C_T>0$ depending on $T$ but independent of $\Delta t$. This proves Corollary~\ref{cor:global-error}.
\end{proof}

\section{Experimental details and additional results}
\label{app:exp-details}
This appendix collects experimental details that are intentionally omitted from the main paper.

\subsection{Datasets, source references, and preprocessing}
\label{app:exp_data}
\textbf{Synthetic dataset.}
The synthetic dataset is generated from the scalar partially observed jump-diffusion
\[
\dd \Theta_t=\kappa(\bar{\theta}-\Theta_t)\dd t+\sigma_\theta \dd W_t^\Theta,
\qquad
\dd X_t=a_1\Theta_t\,\dd t+\sigma_X \dd W_t^X+c_X\dd N_t,
\qquad
\lambda_t=(b_1\Theta_t)^+,
\]
Since the observation law is known by construction, this dataset is mainly used to test latent-state recovery and controlled decoder misspecification.

\textbf{Real datasets.}
We use two real-world datasets to evaluate the proposed method beyond the controlled synthetic setting. The first dataset is a financial time series, XAU/USD, obtained from the Titan FX Research Hub \cite{titanfx_xauusd_2025}  subject to its Historical Data Terms of Use. These terms allow use for market analysis but prohibit redistribution to third parties. We therefore cite the source and describe the preprocessing procedure, but do not redistribute the raw XAU/USD records. The NDBC wave-height data are publicly accessible from NOAA/NDBC, and we provide the station identifier and access information above. We start from one-minute OHLCV records and use the close price trajectory as the observed process, and the full raw trajectory is shown in Figure~\ref{fig:appendix_xau_data}. The second dataset is a non-financial environmental time series, significant wave height (WVHT), collected from NOAA/NDBC buoy station 44027 \cite{ndbc_44027_wvht_2023_2025}; its full raw trajectory is shown in Figure~\ref{fig:appendix_wave_data}.

These two datasets are meant to cover rather different types of dynamics. The XAU/USD series has clear directional movements over long stretches, together with local reversals and sudden price moves, which effectively tests the model's adaptation to latent regime changes and directional shifts. The wave-height series, on the other hand, does not show a clear long-term trend, but has clustered variability, intermittent bursts, and sharp spikes, where the main difficulty is to represent uncertainty under bursty and heteroskedastic behavior.

\begin{figure}[t]
\centering
\begin{subfigure}[t]{0.48\linewidth}
\centering
\includegraphics[width=\linewidth]{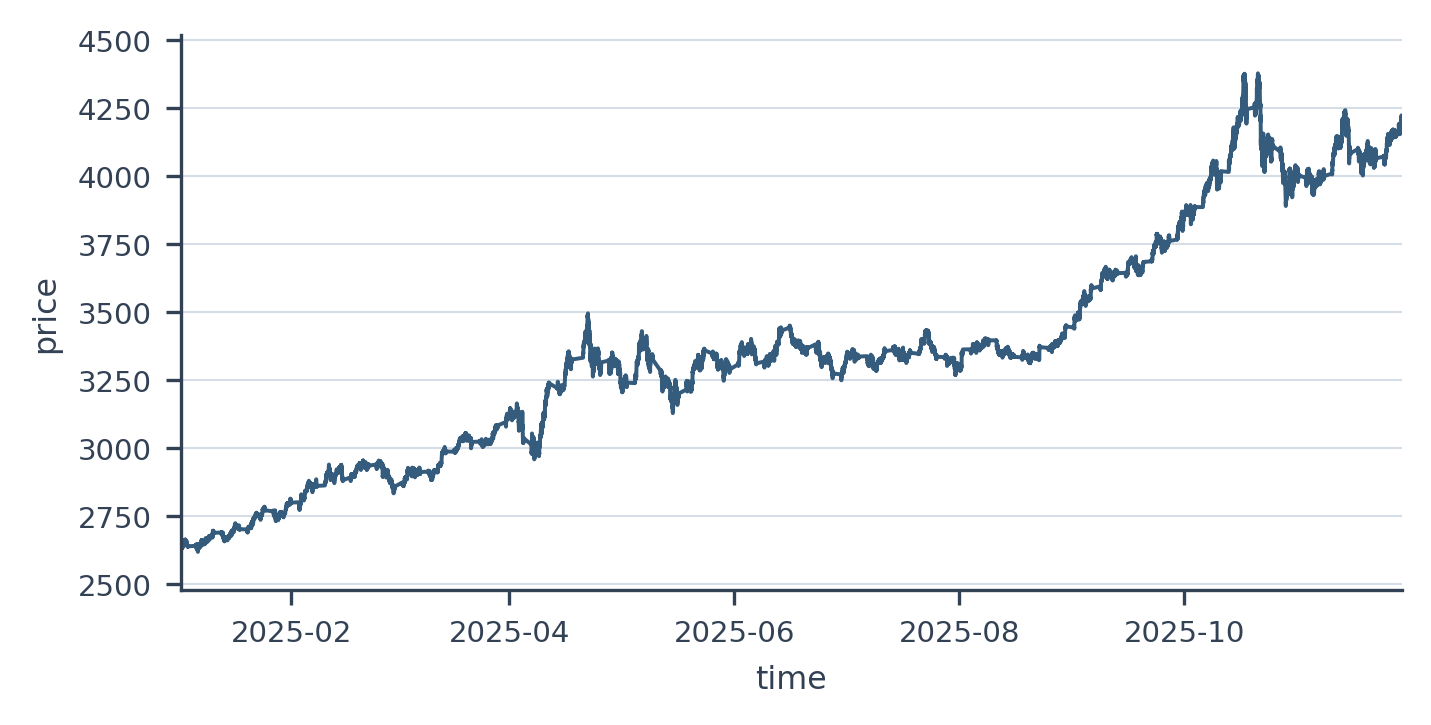}
\caption{Gold price series.}
\label{fig:appendix_xau_data}
\end{subfigure}\hfill
\begin{subfigure}[t]{0.48\linewidth}
\centering
\includegraphics[width=\linewidth]{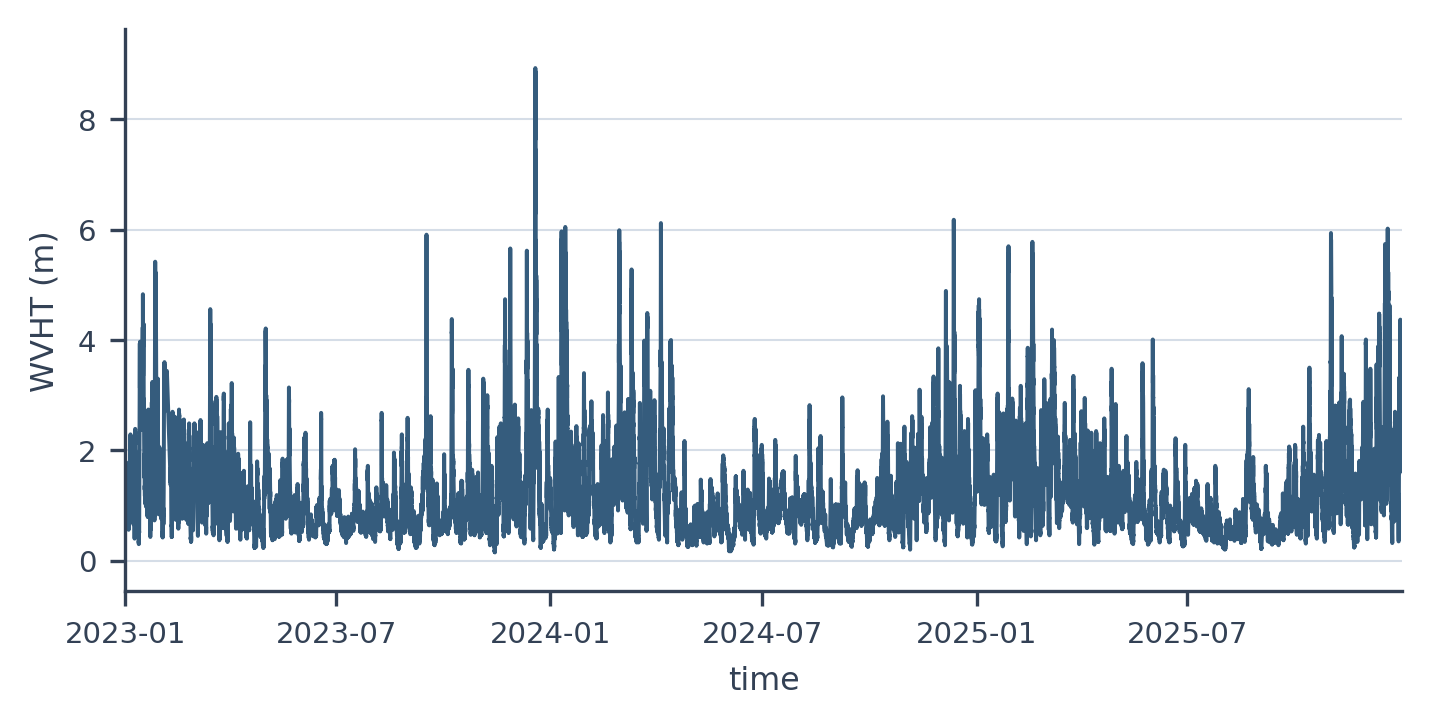}
\caption{Wave-height series.}
\label{fig:appendix_wave_data}
\end{subfigure}
\caption{Full raw series for the two real datasets used in the paper.}
\label{fig:appendix_real_data}
\end{figure}

\textbf{Preprocessing.}
For the synthetic benchmark, we keep the simulated trajectories in their original scale without any transformation. For the real-world datasets, we first extract a single regularly sampled positive series $S_t$. Specifically, we resample the raw one-minute XAU/USD feed to obtain the 10-minute close, and use the WVHT channel directly for the NDBC dataset. We then apply a relative log transformation:
\[
X_t = \log S_t - \log S_0,
\]
where $S_0$ is the first observation of the corresponding series. Next, we construct overlapping context-target windows matching the setup in the main text. All train/validation/test splits are strictly chronological to prevent future data leakage. Note that while Figure~\ref{fig:appendix_real_data} plots the raw series $S_t$ for readability, all model training and evaluation are performed entirely on the transformed series $X_t$.

\subsection{Training protocol and evaluation setup}
\label{app:exp_train}
\renewcommand{\thealgorithm}{B.\arabic{algorithm}}
\textbf{Forecasting setup.}
We do not repeat the window construction and chronological splits here; these are already specified in Section~\ref{sec:experiments} and Appendix~\ref{app:exp_data}. For reproducibility, all experiments use the common forecasting protocol $M=300$, $N=100$, stride $100$, and seed $42$. The time increment is $\Delta t=0.01$ on the synthetic benchmark and $\Delta t=10/1440\approx 0.00694$ days on the real-world datasets. The synthetic series is kept in its native scale, while the real-world datasets are modeled in the relative log-space $X_t=\log S_t-\log S_0$.

\begin{algorithm}[t]
\caption{Training of \textit{Deep ZakaiJ}.}
\label{alg:training}
\small
\begin{algorithmic}[1]
\Require Encoder parameters $\psi$, decoder parameters $\phi$, latent grid $\{\theta_j\}_{j=1}^G$, context length $M$, horizon $N$.
\For{each window $(X_{s:s+M},\, X_{s+M:s+M+N})$}
    \State Initialize $q_0$ to a uniform density on $\{\theta_j\}_{j=1}^G$.
    \For{$k=0,\dots,M-1$}
        \State $q_{k+1} \gets \mathcal{C}_h \circ \mathcal{B}_h \circ \mathcal{A}_{\Delta t} \circ \mathcal{B}_h \circ \mathcal{C}_h(q_k)$ \Comment{Eq.~\ref{eq:strang-update}}
        \State $\pi_{k+1} \gets q_{k+1} / \|q_{k+1}\|_1$;\quad $\beta_{k+1} \gets \sum_j \varphi(\theta_j)\,\pi_{k+1}(\theta_j)\,\Delta\theta$
    \EndFor
    \State Maximize $\mathcal{L}(\phi,\psi)$ w.r.t.\ $\phi,\psi$ via gradient ascent.
\EndFor
\end{algorithmic}
\end{algorithm}

\begin{algorithm}[t]
\caption{Inference of \textit{Deep ZakaiJ}.}
\label{alg:inference}
\small
\begin{algorithmic}[1]
\Require Trained parameters $(\psi,\phi)$, test context $X_{s:s+M}$, horizon $N$, number of rollouts $S$.
\State Run the encoder (Algorithm~\ref{alg:training}, lines 2--6) on $X_{s:s+M}$ to obtain $(q_M, \beta_M, X_{t_M})$.
\For{$m=1,\dots,S$}
    \For{$n=0,\dots,N-1$}
        \State Sample $\theta^{(m)} \sim \pi_{M+n}(\cdot)$ from the normalized belief.
        \State $(\mu_\phi,\,\sigma_\phi,\,\lambda_\phi,\,\gamma_\phi) \gets f_\phi(t_{M+n},\, X_{t_{M+n}}^{(m)},\, \beta_{M+n},\, \theta^{(m)})$
        \State Sample $J_{n+1} \sim \mathrm{Poisson}(\lambda_\phi \Delta t)$; if $J_{n+1} \ge 1$, draw jump size from $\nu^X$.
        \State $X_{t_{M+n+1}}^{(m)} \gets X_{t_{M+n}}^{(m)} + \mu_\phi\,\Delta t + \sigma_\phi\sqrt{\Delta t}\,\xi_{n+1} + \gamma_\phi\, J_{n+1}$, \quad $\xi_{n+1}\sim\mathcal{N}(0,1)$
        \State $q_{M+n+1} \gets \mathcal{A}_{\Delta t}(q_{M+n})$ \Comment{Prior propagation only}
    \EndFor
\EndFor
\State \textbf{Output:} Ensemble $\{X_{t_{M+1}:t_{M+N}}^{(m)}\}_{m=1}^S$.
\end{algorithmic}
\end{algorithm}

\textbf{Deep Zakai architecture.}
A uniform architecture is maintained across both real-world datasets. The observation decoder employs a purely neural formulation, featuring a hidden dimension of $d_{\mathrm{model}}=32$, two transformer encoder layers with two attention heads, and a local-statistics window spanning 20 increments. The network initializing the latent belief possesses a layer configuration of $[8, 64, G]$, where $G$ is the number of latent support points. Each decoder branch for the drift, diffusion, and jump components utilizes a 32-unit hidden layer. The residual correction network within the Zakai update has a hidden width of $128$. 

For the synthetic benchmark, the primary model pairs a linear observation decoder with a 401-point latent grid bounded in $[-2,2]$. The polynomial ablation utilizes $d_{\mathrm{model}}=32$ alongside separate drift, volatility, and jump transformers of depths $(2,1,1)$ and two attention heads. The fully neural ablation mirrors the two-layer, two-head architecture of the real-world configuration.

\textbf{Optimization.}
The model is trained end-to-end in PyTorch using the AdamW optimizer with a batch size of 32. We apply parameter-specific optimization schedules: the residual filter network is updated with a base learning rate of $10^{-3}$ and weight decay of $10^{-4}$, whereas the observation decoder uses a learning rate of $5\times 10^{-4}$ and weight decay of $10^{-3}$. All reported runs are trained for 50 epochs, incorporating a 3-epoch warm-up followed by cosine learning rate decay. Gradient norms are clipped at $10.0$ for the residual filter network and $1.0$ for the decoder.

\textbf{Forecast evaluation.}
During inference, \textit{Deep ZakaiJ} generates 100 Monte Carlo forecast trajectories per window by rolling out the learned one-step jump-diffusion law. Point forecasts are derived from the ensemble path, from which MAE and RMSE are computed. Conversely, CRPS, average log-likelihood, and Cov90 metrics are evaluated directly against the empirical predictive distribution induced by the 100 sampled trajectories. All baseline models are evaluated under identical chronological splits. Among the baselines retained in the main tables, \textit{Chronos} serves as the sole pretrained foundation model. The remaining baselines are trained from their official reference implementations, with input and output lengths constrained to match our protocol.

\textbf{Hardware \& Run-time.}
All \textit{Deep ZakaiJ} runs reported in the paper were executed on a single NVIDIA GeForce RTX 4090 GPU, training and inference time are reported in Table~\ref{app:run_time}.
\begin{table}[H]
\centering
\caption{Runtime of Deep ZakaiJ.}
\label{app:run_time}
\scriptsize
\begin{tabular}{lccc}
\toprule
 & Synthetic & XAU/USD & NDBC Wave-Height \\
\midrule
Training (s) & 638.1 & 1494.8 & 1780.9 \\
Inference (s) & 79.6 & 133.2 & 142.8 \\
\bottomrule
\end{tabular}
\end{table}

\subsection{Additional benchmark tables}
\label{app:exp_benchmark}

Table~\ref{tab:real_world_appendix} confirms that the broader set of baselines does not change the main conclusion. On XAU/USD, \textit{Deep ZakaiJ} remains the strongest model across all metrics. On NDBC, GRU-ODE achieves comparable calibration (Cov90 87.3\%), but \textit{Deep ZakaiJ} retains a clear advantage in CRPS and LogLik. The remaining foundation and sequence models are competitive in point error but fall substantially behind in distributional quality.

\begin{table}[H]
\centering
\caption{Additional baselines on the two real-world datasets.}
\label{tab:real_world_appendix}
\scriptsize
\resizebox{\linewidth}{!}{
\begin{tabular}{lccccc|ccccc}
\toprule
& \multicolumn{5}{c|}{XAU/USD} & \multicolumn{5}{c}{NDBC Wave-Height} \\
\cmidrule(lr){2-6}\cmidrule(lr){7-11}
Model & MAE & RMSE & CRPS & LogLik & Cov90 (\%) & MAE & RMSE & CRPS & LogLik & Cov90 (\%) \\
\midrule
GRU-ODE & 0.0058 & 0.0085 & 0.0042 & 3.56 & 83.5 & 0.3291 & 0.4308 & 0.2457 & -0.61 & \textbf{87.3} \\
S4 & 0.0060 & 0.0089 & -- & -- & -- & 0.3197 & 0.4184 & -- & -- & -- \\
TimesFM & 0.0061 & 0.0092 & 0.0044 & 3.61 & 86.6 & 0.3571 & 0.4928 & 0.2911 & -3.14 & 47.5 \\
Moirai-MoE & 0.0062 & 0.0093 & 0.0045 & 3.59 & 86.0 & 0.3545 & 0.4903 & 0.2895 & -3.27 & 47.7 \\
Lag-Llama & 0.0063 & 0.0095 & 0.0045 & 3.60 & 86.0 & 0.3579 & 0.4937 & 0.2919 & -3.20 & 47.8 \\
iTransformer & 0.0077 & 0.0102 & -- & -- & -- & 0.4092 & 0.5339 & -- & -- & -- \\
\midrule
Deep ZakaiJ & \textbf{0.0056} & \textbf{0.0081} & \textbf{0.0041} & \textbf{3.65} & \textbf{92.9} & \textbf{0.3069} & \textbf{0.4084} & \textbf{0.2175} & \textbf{-0.38} & 85.3 \\
\bottomrule
\end{tabular}}
\end{table}

\subsection{Ablation Study}
\label{app:ablation}
\begin{table}[ht]
\centering
\caption{Ablation results on the effect of latent filtering and model misspecification.}
\label{tab:ablation}
\scriptsize
\setlength{\tabcolsep}{5.0pt}
\begin{tabular}{lccccc}
\toprule
Variant & MAE & RMSE & CRPS & LogLik & Cov90 (\%) \\
\midrule
Deep ZakaiJ Decoder-only & 0.0986 & 0.1441 & 0.0795 & 0.56 & 91.7 \\
Deep ZakaiJ with GRU belief & 0.0930 & 0.1570 & 0.0764 & 0.61 & 86.1 \\
\midrule
Deep ZakaiJ + neural network decoder & 0.0916 & 0.1357 & 0.0704 & 0.32 & 91.1 \\
\midrule
Deep ZakaiJ without $R^{\mathcal A}_\psi$ & 0.0889 & 0.1387 & 0.0675 & 0.69 & 87.2 \\
Deep ZakaiJ without $\mathcal A$ & 0.0860 & 0.1366 & 0.0664 & 0.71 & 86.6 \\
Deep ZakaiJ without $\mathcal B$ & 0.0863 & 0.1364 & 0.0664 & 0.70 & 86.1 \\
Deep ZakaiJ without $\mathcal C$ & 0.0859 & 0.1359 & 0.0662 & 0.72 & 86.8 \\
\midrule
Deep ZakaiJ with $G=51$  & 0.0871 & 0.1371 & 0.0671 & 0.73 & 87.6  \\
Deep ZakaiJ with $G=101$  & 0.0851 & 0.1379 & 0.0660  & 0.61 & 86.2 \\
Deep ZakaiJ with $G=201$  & 0.0868 & 0.1359 & 0.0665 & 0.69 & 85.9 \\
\midrule
Deep ZakaiJ & \textbf{0.0846} & \textbf{0.1344} & \textbf{0.0653} & \textbf{0.78} & \textbf{89.0} \\
\bottomrule
\end{tabular}
\end{table}

We evaluate the contribution of each architectural component on the synthetic benchmark. As shown in Tab.~\ref{tab:ablation}, the ablations are organized into four groups.

\textbf{Filtering necessity.} Removing the Zakai encoder entirely (Decoder-only) degrades all distributional metrics, with CRPS rising from 0.0653 to 0.0795 and LogLik dropping from 0.78 to 0.56. Replacing the structured Zakai filter with a GRU of similar parameter count substantially worsens both point metrics (MAE 0.0930 vs 0.0846) and distributional quality. Together, these confirm that latent filtering is the primary driver of forecast quality, and that the structured splitting update provides a stronger inductive bias than a generic recurrent alternative.

\textbf{Decoder misspecification.} The ground-truth observation law is linear, but replacing the linear decoder with a neural network decoder maintains competitive point accuracy (MAE 0.0916) despite not matching the true dynamics. This suggests that when coupled with the Zakai encoder, a flexible decoder can still recover a reasonable predictive law, demonstrating robustness to model misspecification.

\textbf{Splitting components.} Removing any single substep ($\mathcal{A}$, $\mathcal{B}$, or $\mathcal{C}$) or the learned residual $R_\psi^{\mathcal{A}}$ leads to a consistent drop in Cov90 (from 89.0\% to around 86--87\%), while point metrics remain comparable. This suggests that each substep contributes to calibration even when the marginal effect on point accuracy is small.

\textbf{Grid resolution.} Varying $G$ from 51 to 201 shows stable performance across grid sizes, with the default $G=401$ achieving the best overall balance. 

\section{Limitations}
\label{app:limitation}
The current formulation assumes univariate observations. Scaling to multivariate systems will likely require structured approximations such as tensor decompositions or amortized neural surrogates for the prior propagation step. The encoder relies on a fixed uniform latent grid, whose cost in the $\mathcal{A}$-step and rigid support may limit scalability and expressivity, replacing it with adaptive or particle-based representations could better capture multi-modal regimes. Finally, the filtering framework currently conditions only on past observations of the target series. Incorporating exogenous covariates into the belief update could broaden applicability to domains where external drivers are only partially observed, such as macroeconomic indicators in finance or atmospheric forcing in oceanography.

\section{Impact Statement}
\label{app:impact}
This paper presents a filtering-based framework for forecasting time series with latent regimes and jumps. Potential applications include financial risk management and environmental early warning systems (e.g., for storm-driven wave surges) through improved uncertainty quantification. However, overreliance on model-generated predictive intervals in high-stakes settings carries risks, particularly under distribution shifts or violated assumptions. We advise practitioners to treat these forecasts as complementary tools and to validate their calibration on domain-specific data prior to deployment.

\newpage
\section*{NeurIPS Paper Checklist}

\begin{enumerate}

\item {\bf Claims}
    \item[] Question: Do the main claims made in the abstract and introduction accurately reflect the paper's contributions and scope?
    \item[] Answer: \answerYes{} 
    \item[] Justification: The abstract and introduction accurately reflect our three main contributions, i.e., the Strang-splitting Zakai encoder, the ability of tractable learning, and empirical validation on synthetic and real-world datasets. The supporting theoretical derivations are provided in Sections~\ref{sec:method} and Appendix~\ref{app:proof}, and the empirical results are provided in Section~\ref{app:exp-details}.

\item {\bf Limitations}
    \item[] Question: Does the paper discuss the limitations of the work performed by the authors?
    \item[] Answer: \answerYes{} 
    \item[] Justification: The "Limitations" subsection is included in Appendix~\ref{app:limitation}.

\item {\bf Theory assumptions and proofs}
    \item[] Question: For each theoretical result, does the paper provide the full set of assumptions and a complete (and correct) proof?
    \item[] Answer: \answerYes{} 
    \item[] Justification: For our original theoretical contributions, all underlying assumptions are clearly stated, and complete proofs are provided in the appendix. Meanwhile, results drawn from existing literature are explicitly cited and further discussed in the appendix. 

    \item {\bf Experimental result reproducibility}
    \item[] Question: Does the paper fully disclose all the information needed to reproduce the main experimental results of the paper to the extent that it affects the main claims and/or conclusions of the paper (regardless of whether the code and data are provided or not)?
    \item[] Answer: \answerYes{} 
    \item[] Justification: We have clearly stated the training and forecasting procedures in Section~\ref{sec:method} and Appendix~\ref{app:exp_train}. The datasets details are described in the Appendix~\ref{app:exp_data}. Additionally, we will release our code once the paper is accepted.

\item {\bf Open access to data and code}
    \item[] Question: Does the paper provide open access to the data and code, with sufficient instructions to faithfully reproduce the main experimental results, as described in supplemental material?
    \item[] Answer: \answerYes{} 
        \item[] Justification: We will release our code once the paper is accepted. For the datasets, the synthetic data is generated by a coupled jump-diffusion processes. The real datasets are publicly accessible from their original sources and are cited in the main text and appendix. The NDBC wave-height data are available from NOAA/NDBC. The XAU/USD data are available from the Titan FX Research Hub under its Historical Data Terms of Use, which prohibit redistribution; therefore, we provide the source and preprocessing details but do not redistribute the raw records.

\item {\bf Experimental setting/details}
    \item[] Question: Does the paper specify all the training and test details (e.g., data splits, hyperparameters, how they were chosen, type of optimizer) necessary to understand the results?
    \item[] Answer: \answerYes{} 
    \item[] Justification: We include full procedure for training and forecasting in Section~\ref{sec:method} and Appendix~\ref{app:exp_train}. In Appendix~\ref{app:exp_data}, we describe our dataset preprocessing in full—including the data-splitting strategy and normalization steps—to ensure that our experimental setup can be replicated. Upon acceptance, we will release our source code and reproducible scripts.

\item {\bf Experiment statistical significance}
    \item[] Question: Does the paper report error bars suitably and correctly defined or other appropriate information about the statistical significance of the experiments?
    \item[] Answer: \answerYes{} 
    \item[] Justification: Rather than conventional error bars, we report probabilistic metrics computed over multiple inference runs in Section~\ref{sec:experiments}, which capture the model’s predictive variability.

\item {\bf Experiments compute resources}
    \item[] Question: For each experiment, does the paper provide sufficient information on the computer resources (type of compute workers, memory, time of execution) needed to reproduce the experiments?
    \item[] Answer: \answerYes{} 
    \item[] Justification: We include a detailed runtime comparison Table~\ref{app:run_time}, and the specific GPU types used are listed in the Appendix~\ref{app:exp_train}.
    
\item {\bf Code of ethics}
    \item[] Question: Does the research conducted in the paper conform, in every respect, with the NeurIPS Code of Ethics \url{https://neurips.cc/public/EthicsGuidelines}?
    \item[] Answer: \answerYes{} 
    \item[] Justification: Our work fully conforms with the NeurIPS ethics guidelines, using commercially licensed data responsibly, and ensuring no privacy, fairness, or other concerns arise from our work.

\item {\bf Broader impacts}
    \item[] Question: Does the paper discuss both potential positive societal impacts and negative societal impacts of the work performed?
    \item[] Answer: \answerYes{} 
    \item[] Justification: The broader impacts are discussed as a separate section in Appendix~\ref{app:impact}.
    
\item {\bf Safeguards}
    \item[] Question: Does the paper describe safeguards that have been put in place for responsible release of data or models that have a high risk for misuse (e.g., pre-trained language models, image generators, or scraped datasets)?
    \item[] Answer: \answerNA{} 
    \item[] Justification: The datasets used in this work are either synthetically generated or sourced from publicly available repositories, which do not contain personal or sensitive information. The model itself is a domain-specific time series forecaster and does not pose risks comparable to general-purpose generative models, so no special safeguards are required.

\item {\bf Licenses for existing assets}
    \item[] Question: Are the creators or original owners of assets (e.g., code, data, models), used in the paper, properly credited and are the license and terms of use explicitly mentioned and properly respected?
    \item[] Answer: \answerYes{} 
    \item[] Justification: Our models are clearly cited in both main text and the appendix, and dataset licenses are provided in the appendix.

\item {\bf New assets}
    \item[] Question: Are new assets introduced in the paper well documented and is the documentation provided alongside the assets?
    \item[] Answer: \answerYes{} 
    \item[] Justification: We will release our code once the paper is accepted, which will be documented. For the dataset, the NDBC wave-height data is collected by NOAA and is in the public domain, while the XAU/USD data can be publicly downloaded from the Titan FX Research Hub under its Historical Data Terms of Use.

\item {\bf Crowdsourcing and research with human subjects}
    \item[] Question: For crowdsourcing experiments and research with human subjects, does the paper include the full text of instructions given to participants and screenshots, if applicable, as well as details about compensation (if any)? 
    \item[] Answer: \answerNA{} 
    \item[] Justification: This project does not involve any crowdsourcing or human-subject experiments.

\item {\bf Institutional review board (IRB) approvals or equivalent for research with human subjects}
    \item[] Question: Does the paper describe potential risks incurred by study participants, whether such risks were disclosed to the subjects, and whether Institutional Review Board (IRB) approvals (or an equivalent approval/review based on the requirements of your country or institution) were obtained?
    \item[] Answer: \answerNA{} 
    \item[] Justification: This project does not involve any human-subject experiments.

\item {\bf Declaration of LLM usage}
    \item[] Question: Does the paper describe the usage of LLMs if it is an important, original, or non-standard component of the core methods in this research? Note that if the LLM is used only for writing, editing, or formatting purposes and does \emph{not} impact the core methodology, scientific rigor, or originality of the research, declaration is not required.
    \item[] Answer: \answerNA{} 
    \item[] Justification: This project does not involve LLMs in any core method developments.

\end{enumerate}

\end{document}